\documentclass{article}

\usepackage{arxiv}
\usepackage[utf8]{inputenc}

\usepackage{balance} % for balancing columns on the final page
\usepackage{needspace}

\usepackage{amsfonts}       % blackboard math symbols
\usepackage{verbatim}
\usepackage{multirow}
\usepackage{wrapfig}

\usepackage{amsmath}
\usepackage{amssymb}
\usepackage{amsthm}

\usepackage{pgf,tikz}

\usetikzlibrary{arrows}
\usetikzlibrary{decorations.markings}
\usepackage{accessibility}

\usepackage[round]{natbib}

\usepackage{hyperref}       % hyperlinks
\usepackage{url}            % simple URL typesetting
\usepackage{booktabs}       % professional-quality tables
\usepackage{nicefrac}       % compact symbols for 1/2, etc.
\usepackage{microtype}      % microtypography
\usepackage{graphicx}
\usepackage{doi}

\hypersetup{
    colorlinks=true,       % Enable colored links
    linkcolor=blue,        % Color for internal links (e.g., sections)
    citecolor=cyan,       % Color for citation links
    urlcolor=cyan          % Color for external URLs
}

\newtheorem{theorem}{Theorem}
\newtheorem{example}[theorem]{Example}
\newtheorem{proposition}[theorem]{Proposition}
\newtheorem{lemma}[theorem]{Lemma}

\newtheorem{definition}[theorem]{Definition}
\DeclareMathOperator*{\argmax}{argmax}

\DeclareMathOperator*{\eqdef}{\stackrel{\text{\tiny def}}{=}}
\newcommand\numberthis{\addtocounter{equation}{1}\tag{\theequation}}

\newcommand{\BibTeX}{\rm B\kern-.05em{\sc i\kern-.025em b}\kern-.08em\TeX}

\title{\Large{A View of the Certainty-Equivalence Method for PAC RL as an Application of the Trajectory Tree Method}}

\author{
Shivaram Kalyanakrishnan \\
Department of Computer Science \& Engineering\\
Indian Institute of Technology Bombay\\
Mumbai, 400076 \\
\texttt{shivaram@cse.iitb.ac.in} 
\and
Sheel Shah \\
Department of Electrical Engineering\\
Indian Institute of Technology Bombay\\
Mumbai, 400076 \\
\texttt{19D070052@iitb.ac.in}
\and
Santhosh Kumar Guguloth \\
Department of Computer Science \& Engineering\\
Indian Institute of Technology Bombay\\
Mumbai, 400076 \\
\texttt{santhoshkg@iitb.ac.in}
}

\begin{document}

%%% The following commands remove the headers in your paper. For final 
%%% papers, these will be inserted during the pagination process.

\pagestyle{fancy}
\fancyhead{}
% \fancyhead[L]{PAC RL with a Generative Model}
% \fancyhead[R]{Author Names}
%%% The next command prints the information defined in the preamble.

\maketitle 
\begin{abstract}
Reinforcement learning (RL) enables an agent interacting with an unknown MDP $M$ to optimise its behaviour by observing transitions sampled from $M$. A natural entity that emerges in the agent's reasoning is $\widehat{M}$, the maximum likelihood estimate of $M$ based on the observed transitions. The well-known \textit{certainty-equivalence} method (CEM) dictates that the agent update its behaviour to $\widehat{\pi}$, which is an optimal policy for $\widehat{M}$. Not only is CEM intuitive, it has been shown to enjoy minimax-optimal sample complexity in some regions of the parameter space for PAC RL with a generative model~\citep{Agarwal2020GenModel}.

A seemingly unrelated algorithm is the ``trajectory tree method'' (TTM)~\citep{Kearns+MN:1999}, originally developed for efficient decision-time planning in large POMDPs. This paper presents a theoretical investigation that stems from the surprising finding that CEM may indeed be viewed as an application of TTM. The qualitative benefits of this view are (1) new and simple proofs of sample complexity upper bounds for CEM, in fact under a (2) weaker assumption on the rewards than is prevalent in the current literature. Our analysis applies to both non-stationary and stationary MDPs. Quantitatively, we obtain (3) improvements in the sample-complexity upper bounds for CEM both for non-stationary and stationary MDPs, in the regime that the ``mistake probability'' $\delta$ is small. Additionally, we show (4) a lower bound on the sample complexity for finite-horizon MDPs, which establishes the minimax-optimality of our upper bound for non-stationary MDPs in the small-$\delta$ regime.\end{abstract}

\section{Introduction}
\label{sec:introduction}

The principle of \textit{certainty-equivalence} has been a recurring theme in the design of reinforcement learning (RL) algorithms~\citep{Azar2013GenModel,Agarwal2020GenModel}.
%, which has also been applied in several other contexts~\citep{CE_Theil1957_Note,CE_Johansen1980,CE_Mania2019}. 
%Significant progress has been made on these questions when the learner is given access to a generative model. The setting here is that the agent may sequentially select an arbitrary (state, action) pair or (state, action, time step) triple, and it is provided an accordingly-sampled next state and reward.
Concretely, consider an agent interacting with an unknown Markov Decision Problem (MDP) $M$. The agent gains information about $M$ by repeatedly querying a generative model with an arbitrary (state, action) pair or (state, action, time-step) triple, and it is provided an accordingly-sampled next state and reward. Based on this set of samples $D$, the agent must propose rewarding behaviour for $M$. The first step in applying the certainty-equivalence method (CEM) is to identify $\widehat{M}$, a maximum likelihood estimate of $M$ based on $D$ ($\widehat{M}$ is also called the ``empirical model''). The agent then computes a policy $\widehat{\pi}$ that is optimal for $\widehat{M}$. In other words, the agent computes the same behaviour as it would \textit{if it were certain} that $\widehat{M} = M$. The idea is intuitive since $\widehat{M}$ indeed approaches $M$ as $D$ grows larger. 

A natural question is whether CEM is \textit{optimal} in its sample complexity. A line of work that formalises the problem using the PAC framework has provided partially  affirmative answers, although gaps remain. If $M$ is a stationary MDP, the baseline for comparison has been a sample-complexity lower bound from \citet{Azar2013GenModel}. These authors also provide a sample-complexity upper bound for an iterative implementation of CEM. Their upper bound matches the lower bound when restrictions are placed on some problem parameters---the tolerance $\epsilon$ and the discount factor $\gamma$. In subsequent work, \citet{Agarwal2020GenModel} partially relax the restriction. Interestingly,
\citet{Li_OperationsResearch2023_SampleSizeBarrier} show that minimax-optimality is possible over the full range of problem parameters by injecting randomness into CEM (hence, technically, the resulting algorithm is \textit{not} CEM). They also provide an upper bound for CEM itself in the case that $M$ is a non-stationary MDP, adopting the convention of using a finite horizon $H$ in place of discount factor $\gamma$.
%For comparison, there is no non-trivial PAC RL lower bound for non-stationary MDPs (beyond the one automatically carrying over from stationary MDPs).
Although the preceding analyses~\citep{Azar2013GenModel,Agarwal2020GenModel,Li2020GenModel} vary in approach, they have a common technical core that uses bounds on the variance of the long-term return.

\subsection{Contribution}

In this paper, we provide an alternative perspective on CEM, which offers a new template for analysis and new upper bounds. 

\subsubsection{New analytical framework} We illustrate a connection between CEM and the seemingly-unrelated trajectory tree method (TTM), proposed by \citet{Kearns+MN:1999} for decision-time planning in large MDPs and POMDPs. A trajectory tree is designed to provide unbiased estimates of the value function of \textit{every} possible policy for the task. In TTM,  \citet{Kearns+MN:1999}
deliberately generate several \textit{independent} trajectory trees, so that confident estimates of value functions can be obtained by averaging. Our main insight is that CEM \textit{implicitly} performs the same kind of averaging. Consequently, we can reuse the proof structure accompanying TTM (summarised in Section~\ref{subsec:relatedwork}), only now using a variant of Hoeffding's inequality for a sum of \textit{dependent} random variables~\cite[see Section  5]{Hoeffding:1963}. Otherwise, we only need elementary probability and counting, setting up simple, intuitive proofs. We also obtain quantitative gains.
%, as given below.

\subsubsection{Upper bound for non-stationary MDPs} The more straightforward case for us to analyse is when $M$ is \textit{non-stationary}: that is, its dynamics can change over time. Let CEM-NS denote the algorithm based on the certainty-equivalence principle for this setting. Under CEM-NS, the maximum likelihood MDP $\widehat{M}$ is also a non-stationary MDP, which estimates a separate transition probability distribution over next states for each (state, action, time-step) triple. If $M$ has a set of states $S$, a set of actions $A$, and horizon $H$, our analysis shows that CEM-NS requires $O\left(\frac{|S||A|H^{3}}{\epsilon^{2}} \log \frac{1}{\delta} + \frac{|S|^{2}|A|H^{4} \log|A|}{\epsilon^{2}}\right)$ samples, where tolerance $\epsilon$ and mistake probability $\delta$ are the usual PAC parameters (formally specified in the next section). This bound is in general incomparable with the $O\left(\frac{|S||A|H^{4}}{\epsilon^{2}} \log \frac{|S||A|H}{\delta}\right)$ upper bound shown recently by \citet{Li_OperationsResearch2023_SampleSizeBarrier}, and is tighter by a factor of $H$ in the regime of small $\delta$. Interestingly, our upper bound holds with a weaker assumption on the rewards (explained in Section~\ref{subsec:mdp}) than is common in the literature. We present the main elements of our analytical approach, situated in the context of the non-stationary setting, in Section~\ref{sec:ssw}. 

%Obtaining a cubic dependence on the horizon in the $\delta$-dependent term is significant; the lower bound for stationary MDPs~\citep{Azar2013GenModel} shows that such a dependence is unavoidable (stationary MDPs are special non-stationary MDPs). 

\subsubsection{Upper bound for stationary MDPs} If it is known that $M$ is a stationary MDP, then the certainty-equivalence principle would imply constructing a \textit{stationary} maximum likelihood MDP $\widehat{M}$ by pooling together all the samples for any (state, action) pair. Let CEM-S denote the algorithm that is consistent with this approach. In Section~\ref{sec:stationary}, we analyse CEM-S under the usual assumption that $M$ is infinite-horizon, with discount factor $\gamma < 1$.  A key technical difference emerges when we analyse CEM-S using the TTM toolkit. In the stationary setting, some trajectory trees---or equivalently, ``worlds'', as we shall denote them---use the \textit{same} sample transition at different time steps, and therefore no longer provide \textit{unbiased} value estimates of policies. We simply use the fact that such worlds constitute only a small fraction of the universe of worlds, and hence their influence is limited.
%Atomic worlds $x$ in our set of worlds $X$ are still deterministic, non-stationary, finite horizon MDPs. However, in the stationary setting, $X$ contains a subset $X_{\text{biased}}$ whose elements encode two or more \textit{dependent} samples from $D$. Worlds $x \in X_{\text{biased}}$ do not provide unbiased value-function estimates. We bound the error coming from such worlds by bounding the size of $X_{\text{biased}}$ relative to that of $X$. 

Our eventual sample-complexity upper bound for CEM-S is $\tilde{O}\left(\frac{|S||A|}{(1 - \gamma)^{3}\epsilon^{2}}\left(\log \frac{1}{\delta} + |S||A|\epsilon \right)\right)$, where $\tilde{O}$ suppresses factors that are logarithmic in $\frac{1}{\epsilon}$ and $\frac{1}{1 - \gamma}$. This upper bound matches the lower bound from \citet{Azar2013GenModel} in the regime of small $\delta$. By contrast, the upper bounds provided by \citet{Azar2013GenModel} and \citet{Agarwal2020GenModel} hold for all $\delta \in (0, 1)$, but unlike ours, apply only to restricted ranges of $\epsilon$.

\subsubsection{Lower bound for finite-horizon MDPs} As an independent contribution, we adapt the lower bound of \cite{Azar2013GenModel} to the finite horizon setting, showing that $\Omega\left(\frac{|S||A|H^{3}}{\epsilon^{2}} \log \frac{1}{\delta} \right)$ samples are necessary on some instances for any PAC algorithm in the finite-horizon setting. This result, presented in Section~\ref{sec:lowerbound}, establishes the new finding that within the small-$\delta$ regime, CEM is indeed a minimax-optimal algorithm for non-stationary MDPs.

In short, our paper furthers the understanding of CEM, a natural and intuitive algorithm, by bringing out its connection with TTM, itself a classical algorithm. Our analysis and results are significant to the theory of RL, which is a central paradigm for agent learning. Our work also motivates further analysis and algorithm design.
%Whereas the full range of $\epsilon$ is $(0, \frac{1}{1 - \gamma})$, \cite{Azar2013GenModel} require that $\epsilon \in \left(0, \frac{1}{\sqrt{(1 - \gamma)|S|}}\right]$, and \cite{Agarwal2020GenModel} that $\epsilon \in \left(0, \frac{1}{\sqrt{(1 - \gamma)}}\right]$.
We begin with a formal problem statement (Section~\ref{sec:ps}) and a review of the relevant literature (Section~\ref{subsec:relatedwork}) before presenting our analysis.

\section{PAC RL: Problem Statement}
\label{sec:ps}

%and summarise existing results on the sample complexity required for this problem. We 

We formalise the requirement of PAC RL with a generative model.

\subsection{Markov Decision Problems}
\label{subsec:mdp}

We adopt a definition of MDPs that covers both stationary and non-stationary tasks, with both finite and infinite horizons. An MDP $M = (S, A, T, R, H, \gamma)$ comprises a set of states $S$ and a set of actions $A$. We assume $S$ and $A$ are finite. Positive integer $H$ (possibly infinite) denotes the task horizon; let $[H]$ denote the set $\{0, 1, 2, \dots, H - 1\}$.
The transition function $T: S \times A \times [H] \times S \to [0, 1]$ assigns a probability $T(s, a, t, s^{\prime})$ to change state from $s \in S$ to $s^{\prime} \in S$ by taking action $a \in A$ at time step $t \in [H]$; hence $\sum_{s^{\prime} \in S} T(s, a, t, s^{\prime}) = 1$ for $s \in S$, $a \in A$, $t \in [H]$. Taking action $a \in A$ from state $s \in S$ at time step $t \in H$ also earns a
%non-negative
numeric reward $R(s, a, t)$. Hence, an agent's interaction with the MDP is a sequence $s^{0}, a^{0}, r^{0}, s^{1}, a^{1}, r^{1}, \dots, s^{H - 1}, a^{H - 1}, r^{H - 1}, s^{H}$ wherein for time step $t \in [H]$, the agent (1) takes action $a^{t}$ from state $s^{t}$, (2) obtains reward $r^{t} = R(s^{t}, a^{t}, t)$, and (3) proceeds to state $s^{t + 1} \sim T(s^{t}, a^{t}, t)$, with the convention that $s^{H}$ is a terminal state. The discount factor $\gamma \in [0, 1]$ is used to compute long-term values; we permit $\gamma = 1$ only when $H$ is finite.
%Rather than constrain the reward at each time step, we enforce the weaker r

Previous work~\citep{Azar2013GenModel,Agarwal2020GenModel,Li2020GenModel} has typically assumed that \textit{each} reward comes from a known, bounded range (taken by convention as $[0, 1]$). However, we only enforce the weaker requirement that the discounted \textit{sum} of rewards $\sum_{t \in [H]} \gamma^{t} r^{t}$ lie in a known interval~\citep{JiangAgarwal_OpenProblem_CLOT2018}. For easy comparison with previous results, we take this interval as $[0, V_{\max}]$, where $V_{\max} \leq \min\left\{H, \frac{1}{1 - \gamma}\right\}$, as would follow if each reward is at most $1$. To simplify exposition, we assume
%in Section~\ref{subsec:mdp}
that the rewards are deterministic, and
%in Section~\ref{subsec:learningalgorithms}
% \todo[inline]{should we refer to subsetions 3.1 and 3.2?}
that the reward function is known to the agent. Approximating a stochastic reward function $R$ from samples would not alter the asymptotic complexity of our upper bounds, as also observed by \citet{Agarwal2020GenModel}.

%with its empirical average, provided \eqref{eq:reward-constraint} is satisfied.

%\begin{equation}
%\label{eq:reward-constraint}
%\sum_{t \in [H]} \gamma^{t} r^{t} \leq V_{\max}. 
%\end{equation}
%Observe that $V_{\max}$ cannot exceed either $H$ or $\frac{1}{1 - \gamma}$, and hence can be replaced by one of these quantities when it appears in our upper bounds. 
%As pointed out by \citet{JiangAgarwal_OpenProblem_CLOT2018}, t

%Assuming bounds only on the  The assumption of only the cumulative reward being bounded implements the \textit{sparse reward} setting~\citep{JiangAgarwal_OpenProblem_CLOT2018}, which is a distinctive characteristic of  In sequential decision making tasks with \textit{sparse rewards}

%~\citep{JiangAgarwal_OpenProblem_CLOT2018}

%be at most $1$. In effect, \eqref{eq:reward-constraint} models the regime of \textit{sparse} rewards, which is a distinctive characteristic of sequential decision making.\footnote{It is only to keep our exposition simple that we assume in Section~\ref{subsec:mdp} that the rewards are deterministic, and in Section~\ref{subsec:learningalgorithms} that the reward function is known to the agent. All our claims hold for an agent that approximates a stochastic reward function $R$ with its empirical average, provided \eqref{eq:reward-constraint} is satisfied.}

Let $\pi: S \times [H] \to A$ be a non-stationary policy for $M$. Its value function $V^{\pi}: S \times [H] \to \mathbb{R}$ specifies the expected long-term discounted reward for each $(s, t) \in S \times [H]$, and is given by 
\begin{align*}
V^{\pi}(s, t) &= R(s, \pi(s, t), t) + 
%\\ \nonumber &\phantom{--}
\gamma \sum_{s^{\prime} \in S} T(s, \pi(s, t), t, s^{\prime}) V^{\pi}(s^{\prime}, t + 1),\nonumber
\end{align*}
with the convention that $V^{\pi}(\cdot, H) \eqdef 0$. It is well-known that every MDP has an \textit{optimal} policy $\pi^{\star}: S \times [H] \to A$, which satisfies $V^{\pi^{\star}}(s, t) \geq V^{\pi}(s, t)$ for all $(s, t) \in S \times [H]$ and  $\pi: S \times [H] \to A$. The value function of $\pi^{\star}$ is denoted $V^{\star}$. We may assume $\pi^{\star}$ to be \textit{stationary} (that is, 
independent of time step $t \in [H]$) if $M$ is also stationary (that is, $T$ and $R$ do not depend on $t$) \textit{and} $H$ is infinite.
%SPLIT STATIONARY SENTENCE.

\subsection{Learning Algorithms}
\label{subsec:learningalgorithms}

When learning with a generative model, an algorithm $\mathcal{L}$ can repeatedly query arbitrary $(s, a, t) \in S \times A \times [H]$, and is returned $r = R(s, a, t)$, $s^{\prime} \sim T(s, a, t)$ by the environment. Hence, at any stage, the data $D$ available with the algorithm is the sequence of samples so gathered. Based on $D$, the algorithm may either pick a new tuple to query, or stop and return a policy.

In the PAC formulation, the other inputs to the learning algorithm are a tolerance parameter $\epsilon \in (0, V_{\max})$ and a mistake probability $\delta \in (0, 1)$. The policy $\pi$ returned by $\mathcal{L}$ is $\epsilon$-optimal if for all $s \in S$, $V^{\pi}(s, 0) \geq V^{\star}(s, 0) - \epsilon.$ We require that on every MDP $M$ it is run, $\mathcal{L}$ stop and return an $\epsilon$-optimal policy with probability at least $1 - \delta$. The \textit{sample complexity} of $\mathcal{L}$ on a run is the number of samples it has gathered before termination. In this paper, we restrict our attention to worst case sample-complexity upper bounds (across problem instances) for CEM. For simplicity, we assume that the algorithm  samples each $(s, a, t) \in S \times A \times [H]$ the same number of times $N$, where $N$ is a function of 
$|S|$, $|A|$, $H$, $\gamma$, $V_{max}$, $\epsilon$, and $\delta$. We seek upper bounds on $N$ to ensure the PAC guarantee.
%While the problem specifications above apply to our analysis of both non-stationary and stationary MDPs (in sections \ref{sec:ssw} and \ref{sec:stationary}, respectively),the application to decision-time planning (in Section~\ref{sec:dtp}) involves a minor modification. In this case, $\mathcal{L}$ is provided a ``current state'' $s_{\text{start}} \in S$, and is only required to return $\pi(s_{\text{start}}, 0)$ (which is a single action) for some $\epsilon$-optimal policy $\pi$, with probability at least $1 - \delta$. 
%\gamma^{H} V_{\max} = \frac{epsilon}{2}
%H = \frac{1}{1 - \gamma} \ln(\frac{V_{\max}}{c \epsilon}).
\section{Related Work}
\label{subsec:relatedwork}
In this section, we review sample-complexity bounds for PAC RL, and provide a sketch of TTM.

%To provide a performance guarantee for $\widehat{\pi}$ on $M$, the intermediate step typically involves upper-bounding the divergence between value functions on $\widehat{M}$ and $M$ in terms of the size of $D$~\citep{Azar2013GenModel}. 

\subsection{PAC RL with a Generative Model}
The original PAC formulation of RL was put forth by \citet{Fiechter1994PACRL}, who established that its sample complexity is polynomial in the problem parameters. \citet{Kearns1998PhasedQLearning} then demonstrated that model-free learning algorithms such as $Q$-learning can also achieve polynomial sample complexity. For a stationary, infinite-horizon MDP, the model size scales as $\Theta(|S|^{2}|A|)$, whereas Q-learning uses $\Theta(|S||A|)$ entries. Progress on PAC RL with a generative model has accelerated in the last decade, owing to the minimax-optimal bounds furnished by \citet{Azar2013GenModel}. For stationary, infinite-horizon tasks having $k$ (state, action) pairs, \citet{Azar2013GenModel} show a sample-complexity \textit{lower bound} of
$\Omega\Big(\frac{k}{(1 - \gamma)^{3} \epsilon^{2}} \log\frac{k}{\delta} \Big)$
for obtaining an $\epsilon$-approximation of the optimal action value function $Q^{\star}$.
%The bound holds for for small $\epsilon$, with $\delta$ being the mistake probability.
They construct an MDP instance on which every PAC algorithm must incur at least the specified sample complexity. They also provide an upper bound (applicable to all MDPs), which is is ``minimax-optimal'' in the sense that there exists an MDP on which the lower and upper bounds match up to a constant factor.
%logarithmic factors. 

The tools proposed by \citet{Azar2013GenModel} have been the basis for many subsequent investigations. The essential idea is to construct the empirical model $\widehat{M}$, and to compute an output policy by running value iteration (or policy iteration) on 
$\widehat{M}$ for a finite number of iterations. If $Q_{k}$ is the $k$-step action value function of the output policy on $\widehat{M}$, for $k \geq 1$, the analysis proceeds by inductively upper-bounding the difference between $Q_{k}$ and $Q^{\star}$. Although the original algorithms of \citet{Azar2013GenModel} estimate $Q^{\star}$ with minimax-optimal sample complexity, they do not automatically yield a near-optimal policy. Obtaining such a policy from the action value function would ordinarily
require scaling the sample complexity by $\frac{1}{1 - \gamma}$. A variance-reduction technique proposed by \citet{Sidford2018GenModel}, while different from CEM, directly yields a near-optimal policy without this additional complexity.
%The results described thus far share a technical shortcoming. %First, they assume a model of $M$ in which each reward is bounded in $[0, 1]$, as opposed to our weaker assumption that  the discounted the sum of rewards is bounded in $[0, V_{\max}]$. Incremental feedback is not always available in sequential decision making tasks: in games such as Chess and Go, or robotic applications such as navigation, non-zero reward is administered only when a goal state is reached. A second restriction of 
Yet, the minimax-optimal upper bounds given above do not apply to the entire range of $\epsilon \in (0, V_{\max})$. For instance, the upper bound given by \citet{Azar2013GenModel} only holds for %$\epsilon \in \left(0, \frac{1}{\sqrt{(1 - \gamma)|S|}}\right]$, 
$\epsilon \in (0, 1/ \sqrt{(1 - \gamma) |S| })$, and that of \citet{Sidford2018GenModel} only for $\epsilon \in (0, 1]$.
%These limitations are addressed in two recent additions to the literature, which we summarise next. 
The most recent advance in this line of work is due to \citet{Agarwal2020GenModel}, who
%is especially relevant to ours, since they 
%who furnish a sample complexity analysis of CEM  when applied to stationary MDPs. They 
show that CEM itself can deliver a near-optimal policy for stationary MDPs with minimax-optimal sample complexity, under the constraint that $\epsilon \leq \sqrt{1 / (1 - \gamma)}$. The main components of their 
analysis are bounds on the variance of the return (introduced by \citet{Azar2013GenModel}), and an intermediate MDP designed to break the dependence among samples used to construct the empirical model. In contrast to all these approaches, our analysis only relies on a version of Hoeffding's inequality~\citep{Hoeffding:1963}. We obtain an upper bound for the entire range of problem parameters, whose ratio to the lower bound approaches a logarithmic term as $\delta \to 0$ (while keeping other parameters fixed).

%eventually leads to a constraint involving $|S|$, $|A|$, $\epsilon$, $\gamma$, and $\delta$ (rather than $\epsilon$ and $\gamma$). 
% \todo[inline]{(rather than $\epsilon$ and $r$)? (reward). }

\citet{Li_OperationsResearch2023_SampleSizeBarrier} devise learning algorithms that are minimax-optimal for stationary MDPs for the entire range of parameters, including $\epsilon \in (0, 1 / (1 - \gamma))$ . A key feature of their algorithms is the careful use of randomness for perturbing rewards or action-selection probabilities. The statistical guarantees of these algorithms kick in as soon as the sample size reaches $\Theta (|S||A| / (1 - \gamma) )$, whereas the so-called ``sample barrier'' in the guarantees of \citet{Agarwal2020GenModel} is $\Theta(|S||A| / (1 - \gamma)^{2})$. \citet{Li_OperationsResearch2023_SampleSizeBarrier} do not need to randomise their algorithm for the non-stationary setting, and consequently it boils down to exactly CEM. Our upper bound for CEM in the non-stationary setting is tighter than theirs by a factor of $H$ in the regime of small $\delta$, although it can be looser for large $\delta$.
%(with all other parameters fixed). 

Suppose we wish to estimate the action-value for some $(s, a) \in S \times A$, and this state-action pair gives reward $r$ and transitions to a (random) next state $s^{\prime}$. If the horizon $H$ is finite, then the $H$-step action-value of $(s, a)$ depends only on the $(H - 1)$-step return from $s^{\prime}$. Since our problem does not require us to explicitly estimate $h$-step returns for $h < H$, we make no independent assumption on the range of the $h$-step returns. We allow rewards obtained after visiting $s^{\prime}$ to be arbitrarily large or small (possibly negative), provided the sum of the first $H$ rewards following $(s, a)$ is bounded in $[0, V_{\max}]$. This distinction between the ranges of $H$-step and $(H - 1)$-step rewards becomes inconsequential if $H$ is infinite, and we anyway have to estimate action-values at all states. Mainly focused on stationary, infinite-horizon MDPs, the previous literature~\citep{Azar2013GenModel,Agarwal2020GenModel,Li_OperationsResearch2023_SampleSizeBarrier} constructs concentration bounds by expressing the variance of the return from $(s, a)$ in terms of the variance of the return from $s^{\prime}$. We do not employ such a step. Rather, like in the analysis of TTM, we only apply Hoeffding's inequality to $H$-step returns.

%As we shall observe in Section~\ref{sec:ssw}, the certainty-equivalence approach itself matches all of these results, provided stationary tasks are treated as non-stationary tasks. On non-stationary tasks, our minimax-optimal 

%upper bound improves upon that of \cite{Li2022FiniteH_GenModel_Arxiv} (and also \cite{Yin2021FiniteH_AISTATS}) by a factor of $H$.\shivaram{Prune.}

%Before proceeding, we must take cognisance of 
%There is a much larger body of work related to sample-efficient RL, although not specifically for PAC learning with a generative model. Some related settings include decision-time planning, 

%exploration in continuing tasks without ``reset'' access~\citep{Brafman2002RMAX,Strehl2008MBIE}, the episodic off-policy setting~\citep{Yin2021FiniteH_AISTATS}, and regret minimisation~\citep{Auer2006UCRL,DannBrunskill2017}. 

\subsection{Trajectory Tree Method} In \textit{decision-time planning}~\citep{Kearns2002SparseSampling}, the aim is to identify, a near-optimal action to take from the agent's current state $s^{0}$, with a given probability. A \textit{trajectory tree}~\citep{Kearns+MN:1999} is a randomly-grown tree whose nodes correspond to states, starting with $s^{0}$ at the root. From any node $s^{t}$, $t \in [H]$, exactly one sample $s^{\prime} \sim T(s^{t}, a, t)$
is drawn for each possible action $a \in A$, giving rise to a child node $s^{\prime}$. This process results in a tree of size $|A|^{H}$ (but independent of $|S|$), as illustrated in Figure~\ref{fig:tt}. Each transition has an associated reward. In POMDPs, each node additionally stores a randomly generated \textit{observation}.

%\begin{wrapfigure}{l}{0.5\textwidth}
\begin{figure}[b]
%\vspace{-0.5cm}
\begin{center}
\begin{tikzpicture}[scale=0.3,
every edge/.style={
        draw,
        postaction={decorate,
                    decoration={markings,mark=at position 0.5 with {\arrow[scale=2]{>}}}
                   }
        },
every loop/.style={},
el/.style = {inner sep=2pt, align=left, sloped},
]
%\large
    \tikzset{state/.style={circle,draw=black, minimum size=7mm,thick}}
    \node[state] at (0, 0) (s01){$s^{0}$};
    \node[state] at (-5, -4.4) (s11){$s^{1}_{1}$};
    \node[state] at (5, -4.4) (s12){$s^{1}_{2}$};
    \node[state] at (-8, -8.8) (s21){$s^{2}_{11}$};
    \node[state] at (-2, -8.8) (s22){$s^{2}_{12}$};
    \node[state] at (2, -8.8) (s23){$s^{2}_{21}$};
    \node[state] at (8, -8.8) (s24){$s^{2}_{22}$};
    \path[->]
    (s01) edge [-, thick] node[left] {$a_{1}$}  (s11);
    \path[->]
    (s01) edge [-, thick] node[right] {$a_{2}$}  (s12);
    \path[->]
    (s11) edge [-, thick] node[left] {$a_{1}$}  (s21);
    \path[->]
    (s11) edge [-, thick] node[right] {$a_{2}$}  (s22);
    \path[->]
    (s12) edge [-, thick] node[left] {$a_{1}$}  (s23);
    \path[->]
    (s12) edge [-, thick] node[right] {$a_{2}$}  (s24);
%\normalsize
  \end{tikzpicture}
\end{center}
%\vspace{-0.5cm}
\caption{Example of trajectory tree for horizon $H = 2$, with starting state $s^{0}$, and actions $a_{1}, a_{2}$. Rewards are not shown.}
\label{fig:tt}
%\vspace{-0.5cm}
\end{figure}
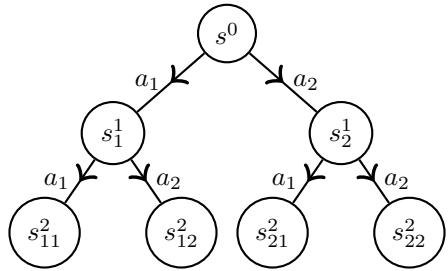
%\end{wrapfigure}

The rationale for building such a tree is that it can provide an unbiased estimate of the value of any arbitrary policy $\pi$ (possibly history-dependent), starting from $s^{0}$. Observe that applying the policy takes us through a trajectory (fixed actions, random next states) with the same probability as in the true MDP or POMDP. Hence, $V^{\pi}(s)$ can be estimated by growing some $m$ independent trajectory trees rooted at $s^{0}$, and averaging their value estimates. Crucially, the same $m$ trees can be used to evaluate every policy $\pi$ from the policy class $\Pi$ being considered (which can be arbitrary). If $\Pi$ is finite, then setting $m = O\left(\frac{V_{\max}}{\epsilon^{2}} \log \frac{|\Pi|}{\delta}\right)$ and selecting the \textit{empirically-best} policy guarantees $\epsilon$-optimality of the chosen action with probability at least $1 - \delta$. This is because, by Hoeffding's inequality, each policy is estimated $\Theta(\epsilon)$-accurately with probability
at least $1 - \frac{\delta}{|\Pi|}$~\cite[see Chapter 6]{Kakade2003PhDThesis}.

TTM essentially arises from a view of any MDP as a distribution over deterministic MDPs (each represented as a trajectory tree from the current state). This same view also facilitates variance reduction in policy search~\citep{Ng2000PEGASUS}.  To obtain bounds independent of $|S|$, \citet{Kearns+MN:1999} \textit{branch} from every action sequence. On the other hand, to analyse CEM, we are happy with bounds that depend on $|S|$. Correspondingly, we represent each deterministic MDP as a collection of samples, one for each (state, action, time-step) triple.
We call such a collection a ``world''.

%apply PAC learning to the task of \textit{decision-time planning}, which can be viewed as an extension of our setting to tasks with very large state spaces. The aim is to obtain sample complexity bounds independent of $|S|$ for PAC action-selection at a given state. 

%Another rich line of research has been on learning algorithms that are constrained to follow a single, continuous trajectory (without the ``reset'' access that a generative model provides)~\citep{Brafman2002RMAX,Strehl2008MBIE}. Yet other formulations require that the agent trade off exploration with exploitation~\citep{Auer2006UCRL,DannBrunskill2017}. It would be interesting to analyse the role of certainty-equivalence in all of these varied settings.

\section{Non-Stationary MDPs}
\label{sec:ssw}

%In this section, w
We present our main ideas
%, %developed 
for the more general setting of non-stationary MDPs. First we summarise CEM in this setting.

%, under which the maximum likelihood MDP $\widehat{M}$ is constructed from the data $D$, and an optimal policy $\widehat{\pi}$ for $\widehat{M}$ computed as the answer. Thereafter we show that $D$ equivalently represents a finite \textit{set} $X$ of deterministic MDPs that are statistical samples of $M$. Introducing a natural definition of value functions to apply to a set of MDPs, we observe that $\widehat{\pi}$ is also optimal for $X$. It emerges as the consequence of a concentration inequality that $\widehat{\pi}$ must also be near-optimal for $M$ with high probability.

\subsection{Certainty-Equivalence: CEM-NS Algorithm}
\label{subsec:ce}

If the underlying MDP $M = (S, A, T, R, H, \gamma)$ is known to be non-stationary, then so is its maximum likelihood estimate $\widehat{M}$. It is sufficient for our purposes to assume that $D$ contains the same number of samples, $N \geq 1$, for each tuple $(s, a, t) \in S \times A \times [H]$. Let $\text{count}(s, a, t, s^{\prime})$ denote the number of observed transitions of $(s, a, t)$ to $s^{\prime} \in S$. The empirical transition function $\widehat{T}$ is set to 
$$\widehat{T}(s, a, t, s^{\prime}) = \frac{\text{count}(s, a, t, s^{\prime})}{N}.$$
$\widehat{M} = (S, A, \widehat{T}, R, H, \gamma)$ is a maximum likelihood estimate of $M$ based on $D$. Let $V^{\star}_{\widehat{M}}: S \times [H] \to \mathbb{R}$ denote the optimal value function of $\widehat{M}$, and let $\widehat{\pi}: S \times [H] \to A$ be a corresponding optimal policy. 
$V^{\star}_{\widehat{M}}$ 
% \todo{$V^{\star}_{\widehat{M}}$?}
and $\widehat{\pi}$ are easily computed by 
%is the unique solution to the following Bellman optimality equations for $\widehat{M}$, and can be computed by 
dynamic programming. 
For $(s, t) \in S \times [H]$,
\begin{align}
&V^{\star}_{\widehat{M}}(s, t) = \max_{a \in A} \Big( R(s, a, t) + \gamma \sum_{s^{\prime} \in S} \widehat{T}(s, a, t, s^{\prime}) V^{\star}_{\widehat{M}}(s^{\prime}, t + 1) \Big);\\
&\widehat{\pi}(s, t) \in \argmax_{a \in A} \Big( R(s, a, t) + \gamma \sum_{s^{\prime} \in S} \widehat{T}(s, a, t, s^{\prime}) V^{\star}_{\widehat{M}}(s^{\prime}, t + 1) \Big).
\label{eq:barmbellmanns}
\end{align}
%Under the certainty-equivalence approach, the answer returned by the algorithm is $\widehat{\pi}$, which sets the action for each $(s, t) \in S \times [H]$ to one that maximises the RHS in \eqref{eq:barmbellman}. 
We denote by CEM-NS (``NS'' for ``non-stationary'') the algorithm that computes $\widehat{\pi}$ as its answer.
%(``NS'' indicates that the unknown MDP $M$ can generally be non-stationary).

\subsection{Set of Worlds}
\label{subsec:sampledmdpsns}

The unknowns in $M$ are the transition probabilities for each $(s, a, t) \in S \times A \times [H]$. Hence the \textit{minimum} amount of information required to build a complete estimate of $M$ is exactly one transition for each $(s, a, t)$ tuple. In our notation, the resulting estimate would be $\widehat{M}$ with $N = 1$---a deterministic MDP that is a ``sample'' of $M$. This estimate would allow the agent to evaluate any arbitrary behaviour, albeit with significant error. The conventional view is that as more transitions are observed, they make the point estimate $\widehat{M}$ more accurate. In our complementary view, larger $N$ simply means more samples of $M$, each sample still an atomic (deterministic) MDP.

Recall that $D$ contains $N$ transitions for  each $(s, a, t) \in S \times A \times [H]$. Take $[N] \eqdef \{1, 2, \dots, N\}$, so each collected transition for $(s, a, t)$ is indexed by some number $i \in [N]$. Let $x \in X \eqdef [N]^{|S||A|H}$ be a string of length $|S||A|H$ on the alphabet $[N]$. We view $x$ as a code specifying a process to construct a deterministic MDP. The input to the process is the random data $D$; hence the resulting MDP $M_{x}$ is a random variable.
%Thus, $x \in X$ serves as a code for a particular deterministic MDP constructed from $D$
Concretely, $x$ picks out a particular transition from the $N$ collected in $D$ for each $(s, a, t)$ tuple. If $x(s, a, t) = i \in [N]$ for some $(s, a, t) \in S \times A \times [H]$, then the transition function $T_{x}$ of $M_{x}$ puts the entire transition probability from  $(s, a, t)$ on the state $s^{\prime} \in S$ observed in the $i$-th sample of $(s, a, t)$.

We refer to each $x \in X$ as a ``world'', defined by the code described above, and specifying a random deterministic MDP $M_{x} = (S, A, T_{x}, R, H, \gamma)$. Thus $X$ is the ``set of all worlds'', of size $N^{|S||A|H}$. For any \textit{fixed} $D$, the collection of $N^{|S||A|H}$ induced MDPs would generally be a multi-set, since multiple worlds $x \in X$ can induce the same MDP. Example~\ref{ex:D} illustrates the definition of $X$ and the process of sampling MDPs from $D$. A world is the semantic counterpart of a trajectory tree, since it allows for any policy to be evaluated. The syntactic difference is that a world associates a sample with every $(s, a, t) \in S \times A \times [H]$, whereas a trajectory tree associates a sample with each (state, action, state, action, $\dots$) sequence visited while constructing the tree.

.
%samples $(s, a, t)$ if $s$ is the next state of some transition $(s^{\prime}, a^{\prime}, t - 1)$, for $s^{\prime} \in S, a^{\prime} \in A$ (or $t = 0$).

\begin{example}
\label{ex:D}
Consider MDP $M$ with states $S = \{s_{0}, s_{1}\}$, actions $A = \{a_{0}, a_{1}\}$, and horizon $H = 3$. The table below describes a possible configuration of data $D$ resulting from sampling each (state, action, time-step) tuple $N = 3$ times.

\begin{center}
%Example of $D$ with $N = 3$\\
%
%\vspace{0.2cm}
%
\begin{comment}
\begin{tabular}{|c|c|c||c|c|}
\hline
\multirow{2}{*}{$s$} &
\multirow{2}{*}{$a$} &
\multirow{2}{*}{$t$} &
\multicolumn{2}{c|}{Samples of $s^{\prime}$} \\ \cline{4-5}
&
&
&
$i = 1$ &
$i = 2$ \\ \hline

$s_{0}$ & 
$a_{0}$ & 
$0$ & 
$s_{1}$ & 
$s_{0}$ \\ \hline

$s_{0}$ & 
$a_{0}$ & 
$1$ & 
$s_{0}$ & 
$s_{0}$ \\ \hline

$s_{0}$ & 
$a_{0}$ & 
$2$ & 
$s_{1}$ & 
$s_{1}$ \\ \hline

$s_{0}$ & 
$a_{1}$ & 
$0$ & 
$s_{1}$ & 
$s_{0}$ \\ \hline

$s_{0}$ & 
$a_{1}$ & 
$1$ & 
$s_{1}$ & 
$s_{1}$ \\ \hline

$s_{0}$ & 
$a_{1}$ & 
$2$ & 
$s_{1}$ & 
$s_{1}$ \\ \hline

$s_{1}$ & 
$a_{0}$ & 
$0$ & 
$s_{0}$ & 
$s_{1}$ \\ \hline

$s_{1}$ & 
$a_{0}$ & 
$1$ & 
$s_{1}$ & 
$s_{1}$ \\ \hline

$s_{1}$ & 
$a_{0}$ & 
$2$ & 
$s_{0}$ & 
$s_{1}$ \\ \hline

$s_{1}$ & 
$a_{1}$ & 
$0$ & 
$s_{1}$ & 
$s_{0}$ \\ \hline

$s_{1}$ & 
$a_{1}$ & 
$1$ & 
$s_{0}$ & 
$s_{0}$ \\ \hline

$s_{1}$ & 
$a_{1}$ & 
$2$ & 
$s_{0}$ & 
$s_{1}$ \\ \hline

\end{tabular}
\end{comment}
\setlength{\tabcolsep}{1.75pt}
%\begin{comment}
%\begin{wraptable}{r}{0.8\textwidth}
\begin{tabular}{|c|c||c|c|c||c|c|c||c|c|c||c|c|c||}
\hline

\multicolumn{2}{|c|}{$s$} & $s_{0}$ & $s_{0}$ & $s_{0}$ & $s_{0}$ & $s_{0}$ & $s_{0}$ & $s_{1}$ & $s_{1}$ & $s_{1}$ & $s_{1}$ & $s_{1}$ & $s_{1}$ \\ \hline

\multicolumn{2}{|c|}{$a$} & $a_{0}$ & $a_{0}$ & $a_{0}$ & $a_{1}$ & $a_{1}$ & $a_{1}$ & $a_{0}$ & $a_{0}$ & $a_{0}$ & $a_{1}$ & $a_{1}$ & $a_{1}$ \\ \hline

\multicolumn{2}{|c|}{$t$} & $0$ & $1$ & $2$ & $0$ & $1$ & $2$ & $0$ & $1$ & $2$ & $0$ & $1$ & $2$ \\ \hline\hline

Samples & $i = 1$ & $s_{1}$ & $s_{1}$ & $s_{1}$ & $s_{1}$ & $s_{1}$ & $s_{1}$ & $s_{0}$ & $s_{1}$ & $s_{0}$ & $s_{1}$ & $s_{0}$ & $s_{0}$ \\ \cline{2-14} 

of & $i = 2$ & $s_{0}$ & $s_{0}$ & $s_{1}$ & $s_{0}$ & $s_{1}$ & $s_{1}$ & $s_{1}$ & $s_{1}$ & $s_{1}$ & $s_{0}$ & $s_{0}$ & $s_{1}$ \\ \cline{2-14}

$s^{\prime}$ & $i = 3$ & $s_{1}$ & $s_{0}$ & $s_{0}$ & $s_{1}$ & $s_{1}$ & $s_{1}$ & $s_{0}$ & $s_{1}$ & $s_{1}$ & $s_{1}$ & $s_{0}$ & $s_{1}$ \\ \hline

\end{tabular}
%\end{wraptable}

%\end{comment}
\end{center}
%For the case of $N = 2$, such a batch can be generated by fixing an arbitrary world $x$, and setting its partner world $x^{\prime}$ to the ``complement'', which selects $1$ whenever $x$ selects $2$, and vice-versa.

\noindent Each sample $i \in [N]$ contains the next state.
%In this case, e
Each world is specified by a $12$-length string over the alphabet $\{1, 2, 3\}$. If we interpret this string in the sequence of the columns in the table, the world $x = 132121123211$ induces MDP $M_{x}$ with transition probabilities $T_{x}(s_{0}, a_{0}, 0, s_{1}) = 1$, $T_{x}(s_{0}, a_{0}, 1, s_{0}) = 1$, $T_{x}(s_{0}, a_{0}, 2, s_{1}) = 1$, and so on. Notice that $x^{\prime} = 122121123211$, which differs from $x$ only in its second position, would induce the same MDP since the second and third samples of $(s_{0}, a_{0}, 1)$ both lead to $s_{0}$. The total number of worlds is $3^{|S||A|H} = 531441$; for $D$ in our example the number of unique MDPs induced is $2^{8} = 256$, since only $8$ of the $12$ $(s, a, t)$ triples have samples with both possible next states. 

\end{example}

\subsection{Evaluating Policies on the Set of Worlds}
\label{subsec:eval}

The value of policy $\pi: S \times [H] \to A$ on MDP $M_{x}$ corresponding to world $x \in X$ is given by
\begin{align}
V^{\pi}_{x}(s, t) &= R(s, \pi(s, t), t) + \gamma \sum_{s^{\prime}} T_{x}(s, \pi(s, t), t, s^{\prime}) V^{\pi}_{x}(s^{\prime}, t + 1)
\label{eqn:ns-bellman}
\end{align}
for $(s, t) \in S \times [H]$. We note $V^{\pi}_{x}$ to be an unbiased estimator of $V^{\pi}$.
\begin{lemma}[Worlds provide unbiased estimates]
\label{lem:ns-exp}
For $x \in X$, $\pi: S \times [H] \to A$, and $(s, t) \in S \times [H]$:
$$\mathbb{E}[V^{\pi}_{x}(s, t)] = V^{\pi}(s, t).$$
\end{lemma}
% \vspace{-1.5cm}

\begin{proof}
Fix $x \in X$ and $\pi: S \times [H] \to A$. As base case of an inductive argument, note that for $s \in S$, $\mathbb{E}[V^{\pi}_{x}(s, H)] \eqdef \mathbb{E}[0] = 0 = V^{\pi}(s, H)$. Assume that for some $t \in [H]$, for $s \in S$, $\mathbb{E}[V^{\pi}_{x}(s, t + 1)] = V^{\pi}(s, t + 1)$.

Now, in \eqref{eqn:ns-bellman}, $T_{x}(s, \pi(s, t), t, s^{\prime})$ is the outcome of a sample for time step $t$,  but the samples for computing $V^{\pi}_{x}(s^{\prime}, t + 1)$ are all from time steps $t + 1$ and higher. Hence, random variables 
$T_{x}(s, \pi(s, t), t, s^{\prime})$ and $V^{\pi}_{x}(s^{\prime}, t + 1)$ are independent, implying that for $s \in S$,
\begin{align*}
\mathbb{E}[V^{\pi}_{x}(s, t)]
&= \mathbb{E}[R(s, \pi(s, t), t)] +\\
&\phantom{aaaaaa} \gamma \sum_{s^{\prime}} \mathbb{E} \Big[ T_{x}(s, \pi(s, t), t, s^{\prime}) \Big] \mathbb{E} \Big[ V^{\pi}_{x}(s^{\prime}, t + 1)\Big] \\
&= R(s, \pi(s, t), t) +
\gamma \sum_{s^{\prime}}  T(s, \pi(s, t), t, s^{\prime})  V^{\pi}(s^{\prime}, t + 1),
\end{align*}
since (1) $T_{x}(s, \pi(s, t), t, s^{\prime})$ is $1$ with probability $T(s, \pi(s, t), t, s^{\prime})$, and otherwise $0$; and (2) from the induction hypothesis, $\mathbb{E}[V^{\pi}_{x}(s, t + 1)] = V^{\pi}(s, t + 1)$ . The RHS is the same as in the Bellman equation on $M$ for $\pi$; hence $\mathbb{E}[V^{\pi}_{x}(s, t)] = V^{\pi}(s, t)$.
\end{proof}
Our upcoming analysis will depend on generalising value functions to \textit{sets} of worlds. We define the value function of a set as the average over its members.
%Since each $x \in X$ is equally likely given $D$, it would be intuitive for the agent to identify a policy that performs well on all $x \in X$. To this end, we define a value function for $\pi: S \times [H] \to A$ on sets of MDPs.
\begin{definition}
For $Z \subseteq X$, $\pi: S \times [H] \to A$, $(s, t) \in S \times [H]$, $$V^{\pi}_{Z}(s, t) \eqdef \frac{1}{|Z|} \sum_{z \in Z} V^{\pi}_{z}(s, t).$$
\end{definition}
%Now, it is not obvious that $X$ must have an optimal policy of the form $\pi: S \times [H] \to A$. In fact, for an arbitrary set of MDPs $X^{\prime}$  that differ only in their transition functions, indeed the optimising policy for a similarly-defined aggregate could exploit both history and stochasticity. However, the particular form of $X$---it contains \textit{every} combination from $D$, equally weighted---removes the need for history or stochasticity to optimise the aggregate value function. As shown below, $\bar{\pi}$ is itself an optimal policy for $X$.
%\begin{lemma}
%For $(s, t) \in S \times [H]$, $$V^{\bar{\pi}}_{X}(s, t) = \max_{\pi: S \times [H] \to A} V^{\pi}_{X}(s, t).$$
%\end{lemma}

\noindent At this point, we can already conceive a recipe based on the classical TTM method to construct a near-optimal policy for $M$. To implement the idea of \citet{Kearns+MN:1999}, consider the $N$-sized subset of worlds $X^{\prime} \subseteq X$, given by $X^{\prime} = \left\{1^{|S||A|H}, 2^{|S||A|H}, \dots, N^{|S||A|H}\right\}$. %Each world in  $X^{\prime}$ serves exactly like a trajectory-tree: allowing for any arbitrary policy to be evaluated, and giving an unbiased estimate of the policy's value on $M$. 
%Since the worlds in $X^{\prime}$ induce i.i.d. samples of $M$, 
By design, no two worlds in $X^{\prime}$ share any samples; hence they can provide $N$ independent value function estimates for each policy. From Hoeffding's Inequality~\citep{Hoeffding:1963}, the value function of each policy would be $\epsilon$-optimal with  probability $1 - \delta/|\Pi|$ for 
$N = O\left( \frac{(V_{\max})^{2}}{\epsilon^{2}} \log \frac{|\Pi|}{\delta}\right)$. Thus, an algorithm that returns an ``optimal policy'' for $X^{\prime}$ from a set of policies $\Pi$ would meet our PAC criterion with about $O\left( \frac{|S| |A| H(V_{\max})^{2}}{\epsilon^{2}} \log \frac{|\Pi|}{\delta}\right)$ samples~\citep{Kearns+MN:1999,Kakade2003PhDThesis}. Unfortunately, it is not easy to \textit{compute} an optimal policy for $X^{\prime}$ if the policy class $\Pi$ is the (usual) set of Markovian, non-stationary policies. The structure of $X^{\prime}$ is such that in general, \textit{history-dependent} policies can perform strictly better than Markovian policies.
%Unfortunately, an optimal policy for $X^{\prime}$, in the sense that it must dominate all other policies for all $(s, t) \in S \times [H]$, would in general be \textit{history-dependent}, and potentially expensive to compute from $D$.
On the other hand, it is straightforward to compute an optimal Markovian, non-stationary policy for the entire universe of worlds $X$. In fact, as formalised in the following lemma, value functions of policies turn out to be \textit{identical} on $\widehat{M}$ and on $X$. Therefore, the output of CEM-NS---$\widehat{\pi}$---is itself an optimal policy for $X$!
\begin{lemma}[Consistency of $X$ and $\widehat{M}$]
\label{lem:XequalsbarmMNS}
For $\pi: S \times [H] \to A$ and $(s, t) \in S \times [H]$, $$V^{\pi}_{X}(s, t) = V^{\pi}_{\widehat{M}}(s, t).$$
\end{lemma}
%\begin{corollary}
%\label{cor:Xopt}
%For $(s, t) \in S \times [H]$ and $\pi: S \times [H] \to A$, $V^{\widehat{\pi}}_{X}(s, t) \geq V^{\pi}_{X}(s, t).$
%\end{corollary}
The proof of this important lemma is given in Appendix~\ref{app:proofsofXequalsM}. The idea is to expand $V^{\pi}_{X}$ and use the fact that each sample in $D$ occurs in exactly the same number of worlds $x \in X$, whereupon it emerges that $V^{\pi}_{X}$ satisfies the Bellman equations for $\pi$ on $\widehat{M}$.

The crux of our paper is in the contrast between $X^{\prime}$ and $X$. Although $V^{\pi}_{x}$ is an unbiased estimate of $V^{\pi}$ for each $x \in X$ and $\pi: S \times [H] \to A$, the deviation of their \textit{average} $V^{\pi}_{X}$ from $V^{\pi}$ cannot be bounded directly using Hoeffding's inequality, since $V^{\pi}_{x}$ and $V^{\pi}_{x'}$ could be \textit{dependent} for worlds $x, x^{\prime} \in X$. For example, the worlds $1^{|S||A|H}$ and $12^{|S||A|H - 1}$ use the same sample for $(s_{0}, a_{0}, 0)$. In spite of this dependence, can we still piggyback on the analytical framework of TTM? Our answer is affirmative, and forms the basis of our view of CEM as an application of TTM.

%The advantage offered by $X^{\prime}$ is that it contains mutually independent worlds; the disadvantage is the difficulty of computing an optimal policy over these worlds. In contrast, it is easy to compute an optimal policy for $X$ (in fact it is $\hat{\pi}$), but does $X$ also contain mutually independent worlds? It does not

%Unlike $X^{\prime}$, worlds in $X$ may share samples, hence yield \textit{dependent} samples of $M$. We are left to work out whether optimising behaviour for $X$ can yet be as rewarding as optimising it over i.i.d. samples of $M$, such as in $X^{\prime}$. We proceed to an affirmative result.

%The set of sampled worlds $X$, established as equivalent to $\widehat{M}$ for policy evaluation, becomes a convenient tool to show the minimax-optimality of the \textbf{CE-NS} algorithm.

%Now, it is not obvious that $X$ must have an optimal policy of the form $\pi: S \times [H] \to A$. In fact, for an arbitrary set of MDPs $X^{\prime}$  that differ only in their transition functions, indeed the optimising policy for a similarly-defined aggregate could exploit both history and stochasticity. However, the particular form of $X$---it contains \textit{every} combination from $D$, equally weighted---removes the need for history or stochasticity to optimise the aggregate value function. As shown below, $\bar{\pi}$ is itself an optimal policy for $X$.

\subsection{Batches of Mutually-Disjoint Worlds}
\label{subsec:convergenceoptimal}

%Our first observation is that for each world $x \in X$ and policy $\pi: S \times [H] \to A$, the value function $V^{\pi}_{x}$ is an \textit{unbiased} estimate of $V^{\pi}_{M}$.

%The independence of the estimated transition probability and the estimated future rewards, which allows us to claim $V^{\pi}_{x}$ as an unbiased estimate of $V^{\pi}_{M}$, is the core element of our proof. 

%x can evaluate each policy. 

%The independence of the estimated transition probability and the estimated future rewards no longer holds when the MDP known to be stationary, as we shall discuss in Section~\ref{sec:stationary}.  Even for the non-stationary case, we must confront the dependence arising from a different source. 

%While it is handy that each $x \in X$ can provide an unbiased estimate of the value function, observe that the MDPs corresponding to different elements $x \in X$ need not be independent.
%Although $X$ does not contain independent samples of $M$

%Although $V^{\pi}_{x}$ is an unbiased estimate of $V^{\pi}$ for each $x \in X$ and $\pi: S \times [H] \to A$, the deviation of their average $V^{\pi}_{X}$ from $V^{\pi}$ cannot be bounded using Hoeffding's inequality, since $V^{\pi}_{x}$ and $V^{\pi}_{x'}$ could be \textit{dependent} for worlds $x, x^{\prime} \in X$. For example, the worlds $1^{|S||A|H}$ and $12^{|S||A|H - 1}$ use the same sample for $(s_{0}, a_{0}, 0)$. To proceed, 

We consider $N$-sized ``batches'' within $X$ that do lead to independent samples of $M$. Define worlds $x, x^{\prime} \in X$ to be \textit{disjoint} if for all $(s, a, t) \in S \times A \times [H]$,  $x(s, a, t) \neq x'(s, a, t)$. In other words, $x$ and $x^{\prime}$ are disjoint if they do not share any samples. A \textit{batch} $b \subseteq X$ is a set of some $N$ mutually disjoint elements of $X$. The set 
$\{132212312132, 221323121321, 313131233213\}$ is a batch in Example~\ref{ex:D}, as also is set $X^{\prime}$ from Section~\ref{subsec:eval}. For $x \in X$, let $B_{x}$ be the set of all batches in which $x$ is present, and let $B$ be the set of all batches. Simple counting (provided in Appendix~\ref{app:batch-counting}) shows that for $x \in X$, $|B_{x}| = (N - 1)!^{|S||A|H - 1}$, and $|B| = N!^{|S||A|H - 1}$. Recall that $V^{\pi}_{X}$ is the average value function of $\pi$ over worlds $x \in X$. At the heart of our proof is the following equation, which shows $V^{\pi}_{X}$ also as the average of the value functions of $\pi$ over batches $b \in B$.
\begin{align}
V^{\pi}_{X}(s, t) &= \frac{1}{|X|} \sum_{x \in X} V^{\pi}_{x}(s, t)
= \frac{1}{|X|} \sum_{b \in B} \sum_{x \in b} \frac{1}{|B_{x}|} V^{\pi}_{x}(s, t) \nonumber \\
&= \frac{N}{|X|(N - 1)!^{|S||A|H - 1} } \sum_{b \in B} \frac{\sum_{x \in b} V^{\pi}_{x}(s, t)}{N} \nonumber \\
&= \frac{1}{|B|} \sum_{b \in B} V^{\pi}_{b}(s, t).
\label{eqn:XasBsum}
\end{align}
%where $V^{\pi}_{b}$ is the average of the value functions on the MDPs present in $b$ (similar to the definition for $X$ in \eqref{eqn:Xvalue}).
%We again use the fact that $X$ contains every possible combination from $D$, which makes it symmetric with respect to each element.
The significance of \eqref{eqn:XasBsum} is that for \textit{each} batch $b \in B$, $V^{\pi}_{b}$ is indeed an average of $N$ \textit{independent} random variables, whose deviation from their expected value can be bounded using Hoeffding's inequality. Since $V^{\pi}_{X}$ is a convex combination of $V^{\pi}_{b}, b \in B$, we can apply Hoeffding's (less-used) result on the sums of dependent random variables~\cite[see Section 5]{Hoeffding:1963}. We restate Hoeffding's result as the following lemma. The commonly-used version of Hoeffding's inequality for independent random variables~\cite[see Theorem 2]{Hoeffding:1963} is obtained by taking $m = 1$.

\begin{lemma}
\label{lem:dephoeffding}[Hoeffding's inequality for average of certain dependent random variables]
Fix positive integers $\ell$ and $m$. For $i \in \{1, 2, \dots \ell\}$, $j \in \{1, 2, \dots, m\}$, let $U_{i, j}$ be a real-valued random variable supported on $[\alpha, \beta] \subset \mathbb{R}$; suppose $U_{i, j}$ and $U_{i, j^{\prime}}$ are independent for $j, j^{\prime} \in \{1, 2, \dots, m\}$ if $j \neq j^{\prime}$. Note that $U_{i, j}$ and $U_{i^{\prime}, j^{\prime}}$ could be dependent if $i \neq i^{\prime}$, for $i, i^{\prime} \in \{1, 2, \dots, \ell\}$, and $j, j^{\prime} \in \{1, 2, \dots, m\}$. Define $$U_{i} \eqdef \frac{1}{m} \sum_{j = 1}^{m} U_{i, j} \text{   and   }
U \eqdef \sum_{i = 1}^{\ell} p_{i} U_{i}$$ for some $p_{1}, p_{2}, \dots, p_{\ell} \in [0, 1]$ satisfying $\sum_{i = 1}^{\ell}$ $p_{i} = 1$.
%Then, f
For $\gamma > 0$, 
\begin{align*}
\mathbb{P}\{U \geq \mathbb{E}[U] + \gamma\} &\leq \exp \left( \frac{-2 m \gamma^{2}}{(\beta - \alpha)^{2}} \right) \text{ and } \\
\mathbb{P}\{U \leq \mathbb{E}[U] - \gamma\} &\leq \exp \left( \frac{-2 m \gamma^{2}}{(\beta - \alpha)^{2}} \right).
\end{align*}
\end{lemma}
For convenient reference, we give a proof of this lemma in Appendix~\ref{app:dephoeffdingproof} (the original proof is from \citet[see Section 5]{Hoeffding:1963}). We are ready for our main result, which uses Lemma~\ref{lem:dephoeffding} to legitimise CEM's approach of optimising behaviour uniformly over every possible batch, in contrast with TTM's approach of doing so for a single, arbitrary batch.
%, CEM optimises it uniformly over every possible batch.

\begin{comment}
%Considering that $X$ (which contains dependent worlds $x$) is a union of $N$-sized batches $b$ (each of which contains mutually independent worlds from $X$), we bound the deviation of $V^{\pi}_{X}$ from $V^{\pi}_{M}$ by individually bounding the deviations of each $V^{\pi}_{b}$, $b \in B$, from $V^{\pi}_{M}$.
\begin{lemma}
\label{lem:Xdeviation}
For $\pi: S \to A$, $(s, t) \in S \times [H]$, $\epsilon \in (0, V_{\max})$,
\begin{align*}
\mathbb{P}\{
V^{\pi}(s, t) -
V^{\pi}_{X}(s, t) > \epsilon\} &\leq \exp\left(- \frac{2N\epsilon^{2}}{(V_{\max})^{2}}\right);\\
\mathbb{P}\{V^{\pi}_{X}(s, t) -
V^{\pi}(s, t) > \epsilon\} &\leq  \exp\left(- \frac{2N\epsilon^{2}}{(V_{\max})^{2}}\right).
\end{align*}
\end{lemma}
\begin{proof}
Each $b \in B$ is an $N$-sized subset of $X$ containing i.i.d. samples of $M$. Since $V^{\pi}_{M}(s, t)$ lies in $(0, V_{\max})$ for $s \in S$, Hoeffding's inequality~\cite[see Theorem 2]{Hoeffding:1963}  yields for each $b \in B$: $$\mathbb{P}\{V^{\pi}(s, t) -
V^{\pi}_{b}(s, t) > \epsilon\} \leq \exp\left(- \frac{2N\epsilon^{2}}{(V_{\max})^{2}}\right);$$
$$\mathbb{P}\{ V^{\pi}_{b}(s, t) - V^{\pi}(s, t) > \epsilon\} \leq \exp\left(- \frac{2N\epsilon^{2}}{(V_{\max})^{2}}\right).$$ Since $V^{\pi}_{X} = \frac{1}{|B|} \sum_{b \in B} V^{\pi}_{b}$,  Hoeffding's bound on sums of dependent random variables~\cite[see Section 4]{Hoeffding:1963} delivers the claim of the lemma.
\end{proof}
    
\end{comment}
%We are ready for our main result.
\begin{theorem}[Sample complexity of CEM-NS]
\label{thm:non-stat}
The CEM-NS algorithm provides the relevant PAC guarantee for non-stationary MDP $M$ with parameters $\epsilon \in (0, V_{
max})$, $\delta \in (0, 1)$ if run with
$$N = \left\lceil \frac{2(V_{\max})^{2}}{\epsilon^{2}} \ln \frac{|S||A|^{|S|H}}{\delta}\right\rceil.$$
\end{theorem}
\begin{proof}
Recall that CEM-NS returns $\widehat{\pi}$,  which depends on the data $D$, and hence is random. Lemma~\ref{lem:XequalsbarmMNS} gives us that $\widehat{\pi}$ is optimal for $X$. Now, if $\widehat{\pi}$ is not $\epsilon$-optimal for $M$, it means that either (i) $X$ under-estimates  $V^{\pi^{\star}}(s, 0)$ by at least $\frac{\epsilon}{2}$, or (ii) $X$ over-estimates $V^{\pi}(s, 0)$ by at least $\frac{\epsilon}{2}$ for some non-$\epsilon$-optimal policy $\pi: S \times [H] \to A$ and state $s \in S$. From \eqref{eqn:XasBsum}, we have that $V^{\pi}_{X}(s, 0) = \sum_{b \in B} \frac{1}{|B|}V^{\pi}_{b}(s, 0)$, where $V^{\pi}_{b}(s, 0)$ for each $b \in B$ is a sum on $N$ independent random variables with mean $V^{\pi}(s, 0)$ (from Lemma~\ref{lem:ns-exp}). Define $\delta^{\prime} \eqdef \frac{\delta}{|S||A|^{|S|H}}$. We apply Lemma~\ref{lem:dephoeffding} to get  $\mathbb{P}\left\{V^{\pi^{\star}}_{X}(s, 0) \leq V^{\pi^{\star}}(s, 0) - \frac{\epsilon}{2}\right\} \leq \delta^{\prime}$ and $\mathbb{P}\left\{V^{\pi}_{X}(s, 0) \geq V^{\pi}(s, 0) + \frac{\epsilon}{2}\right\} \leq \delta^{\prime}$
for $s \in S$ and $\pi: S \times {H} \to A$. Since there are $|S|$ states and $|A|^{|S|H}$ policies, a union bound establishes that $\widehat{\pi}$ is $\epsilon$-optimal with probability at least $|S||A|^{|S|H} \delta^{\prime} = \delta$.
\end{proof}
Since each $(s, a, t) \in S \times A \times [H]$ is sampled $N$ times by CEM-NS, and since $V_{\max} \leq H$, the algorithm's overall sample complexity is
$$O\left(\frac{|S||A|H^{3}}{\epsilon^{2}} \left(\log \frac{1}{\delta} + |S|H \log |A|\right)\right).$$ Recall that \citet{Li_OperationsResearch2023_SampleSizeBarrier} show a bound of $\tilde{O}\left(\frac{|S||A|H^{4}}{\epsilon^{2}} \log \frac{1}{\delta}\right)$ samples for \textbf{CEM-NS}. In the regime that $\delta$ is made small after fixing other parameters, our bound is tighter by a factor of $H$. This is a significant result since the coefficient of $\log \frac{1}{\delta}$ now has a cubic dependence on the horizon---which we show is unavoidable by providing an explicit lower bound in Section~\ref{sec:lowerbound}.

\section{Stationary MDPs}
\label{sec:stationary}

In this section, we analyse CEM when applied to \textit{stationary} MDP $M$. We now use $\widehat{M}$, $\widehat{\pi}$, and $X$ to denote corresponding objects in the stationary setting. 

\subsection{Certainty-Equivalence: CEM-S Algorithm}
\label{subsec:consistent}

We continue with the same definition of $M = (S, A, T, R, H, \gamma)$, only now assuming that $T$ and $R$ do not depend on the time step $t$ (which we drop from our notation). Consistent with previous literature, we also assume $H = \infty$. Since there is no time-dependence, we take that each tuple $(s, a) \in S \times A$ is sampled $N$ times, $N \geq 1$, in the data $D$. For $(s, a, s^{\prime}) \in S \times A \times S$, let $\text{count}(s, a, s^{\prime})$ denote the number of transitions observed in $D$ to reach $s^{\prime}$ by taking $a$ from $s$. The empirical MDP $\widehat{M} = (S, A, \widehat{T}, R, H, \gamma)$ therefore satisfies $$\widehat{T}(s, a, s^{\prime}) = \frac{\text{count}(s, a, s^{\prime})}{N}$$ 
for $s, s^{\prime} \in S, a \in A$. It is well-known that every stationary, infinite-horizon MDP admits a deterministic optimal policy. The optimal value function $V^{\star}_{\widehat{M}}$ and optimal policy $\widehat{\pi}: S \to A$ satisfy 
\begin{align*}
V^{\star}_{\widehat{M}}(s) &= \max_{a \in A} \Big( R(s, a) + \gamma \sum_{s^{\prime} \in S} \widehat{T}(s, a, s^{\prime}) V^{\star}_{\widehat{M}}(s^{\prime}) \Big); \\
\widehat{\pi}(s) &\in \argmax_{a \in A} \Big( R(s, a) + \gamma \sum_{s^{\prime} \in S} \widehat{T}(s, a, s^{\prime}) V^{\star}_{\widehat{M}}(s^{\prime}) \Big)
%\label{eq:barmbellmans}
\end{align*}
for $s \in S$. 
$V^{\star}_{\widehat{M}}$ and $\widehat{\pi}$ may be computed from $D$ by value iteration, policy iteration, or linear programming~\citep{Littman+DK:1995}. Our upcoming sample-complexity bound would only get scaled by a constant factor, say, if an $\frac{\epsilon}{2}$-optimal policy is computed for $\widehat{M}$---and this computation needs only a polynomial number of arithmetic operations in $|S|$, $|A|$, $\frac{1}{1 - \gamma}$, and $\log\frac{1}{\epsilon}$. To keep the exposition uncluttered, we assume that our certainty-equivalence implementation---denoted CEM-S (``S'' for ``stationary'')---indeed computes and returns $\widehat{\pi}$ exactly.

\subsection{Truncated Horizon}
\label{subsec:truncatedhorizon}

The assumption of a finite horizon $H$ in the non-stationary setting meant that our worlds would also have this same horizon $H$. Since we have taken $H = \infty$ for stationary $M$, we require an intermediate step to apply the framework of a set of worlds. Consider a finite horizon MDP $M_{\overline{H}} = (S, A, T, R, \overline{H}, \gamma)$ that is identical to $M$ other than for having a \textit{finite} horizon $\overline{H} \eqdef \left\lceil\frac{1}{1 - \gamma} \ln \left (\frac{4 V_{\max}}{\epsilon}\right)\right\rceil$. The corresponding empirical MDP is  $\widehat{M}_{\overline{H}} = (S, A, \widehat{T}, R, \overline{H}, \gamma)$. Since the infinite-discounted sum from each state is constrained to $[0, V_{\max}]$, 
%Since we have assumed rewards to be non-negative, and $\overline{H}$ is sufficiently large, 
the truncation loss 
$\mathbb{E}_{\pi}[\sum_{t = \overline{H}}^{\infty} \gamma^{t} r^{t}]$
must lie in $[0, \gamma^{\overline{H}} V_{\max}] \subseteq [0, \frac{\epsilon}{4}]$
%is non-negative, and at most $\gamma^{\overline{H}} V_{\max} \leq \frac{\epsilon}{4}$ 
on both $M$ and $\widehat{M}$.
\begin{proposition}[Bounded truncation loss]
\label{prop:trunc}
For $\pi: S \to A$, $s \in S$: $$V^{\pi}(s) - \frac{\epsilon}{4}
\leq 
V^{\pi}_{M_{\overline{H}}}(s, 0) 
\leq V^{\pi}(s); V^{\pi}_{\widehat{M}}(s) - \frac{\epsilon}{4} \leq V^{\pi}_{\widehat{M}_{\overline{H}}}(s, 0) \leq V^{\pi}_{\widehat{M}}(s).$$
\end{proposition}
The truncated horizon $\overline{H}$ has no relevance to the CEM-S \textit{algorithm} itself; the algorithm is based on the true (infinite) horizon of $M$. However, our  \textit{analysis} works on $M_{\overline{H}}$, 
using worlds of length $|S||A|\overline{H}$ to encode samples. Proposition~\ref{prop:trunc} enables us to relate the extent of sub-optimality over a finite horizon $\overline{H}$ with that on an infinite horizon.

%, we observe that if \textbf{CE-S} is $\frac{3\epsilon}{4}$-optimal on $M_{\overline{H}}$, then it must be $\frac{\epsilon}{2}$-optimal on $M$. \textcolor{red}{$\blacksquare$}.

%itself has no 

%does not use $\overline{H}$: it uses the horizon of $M$, which is infinite.

\subsection{Set of Worlds}
\label{subsec:sampledmdpss}

%Our analysis based on the set of sampled worlds takes each world to be a deterministic, finite-horizon MDP sampled from $M$. The next state for $(s, a, t) \in S \times A \times [H]$ in a world can now come from the samples of any tuple $(s, a, t^{\prime})$ for $t^{\prime} \in [H]$. 

%In keeping with the finite horizon of $\widehat{M}_{\overline{H}}$, 

Each world $x$ in our set of worlds $X$ is an $|S||A|\overline{H}$-length string on the alphabet $[N]$. It associates a transition sample from $D$ for each $(s, a, t) \in S \times A \times [\overline{H}]$. However, since $M_{\overline{H}}$ is stationary, samples are not distinguished based on time step in the data $D$. Hence, if $D$ in Example~\ref{ex:D} had come from a \textit{stationary} MDP, we would ignore $t$ and pool together all $N = 9$ samples for each (state, action) pair. Thus, for the pair $(s_{0}, a_{0})$, the sequence of samples (read row by row from top to bottom, and left to right within each row) would be $s_{1}, s_{1}, s_{1}, s_{0}, s_{0}, s_{1}, s_{1}, s_{0}, s_{0}$. The world $x = 571634978542$ would induce a deterministic MDP with probabilities of $1$ for the twelve transitions
$(s_{0}, a_{0}, 0, s_{0})$,
$(s_{0}, a_{0}, 1, s_{1})$,
$(s_{0}, a_{0}, 2, s_{1})$,
$(s_{0}, a_{1}, 0, s_{1})$,\\
$(s_{0}, a_{1}, 1, s_{1})$,
$(s_{0}, a_{1}, 2, s_{0})$,
$(s_{1}, a_{0}, 0, s_{1})$, 
$(s_{1}, a_{0}, 1, s_{0})$,
$(s_{1}, a_{0}, 2, s_{1})$,
$(s_{1}, a_{1}, 0, s_{0})$, 
$(s_{1}, a_{1}, 1, s_{0})$, and 
$(s_{1}, a_{1}, 2, s_{0})$. In general there are $N^{|S||A|\overline{H}}$ worlds $x \in X$.

%For instance, one possible world $x \in X$ for the data $D$ provided in Example~\ref{ex:D} is $3416533243$. With the convention that $NH$ samples are enumerated for each $(s, a) \in S \times A$ row-wise, left-to-right, this world $x$ satisfies

In the stationary setting, it is seen that $X$ evaluates policies identical to $\widehat{M}_{\overline{H}}$. 
\begin{lemma}[Consistency of $X$ and $\widehat{M}_{\overline{H}}$]
\label{lem:XequalsbarmMS}
For $\pi: S \times [\overline{H}] \to A$ and $(s, t) \in S \times [\overline{H}]$, $$V^{\pi}_{X}(s, t) = V^{\pi}_{\widehat{M}_{\widehat{H}}}(s, t).$$
\end{lemma}
The proof of Lemma~\ref{lem:XequalsbarmMS} is identical to that of Lemma~\ref{lem:XequalsbarmMNS}, and given in Appendix~\ref{app:proofsofXequalsM}. The lemma and upcoming results also apply to stationary policies (the ``$t$'' comes from the finite horizon of $\widehat{M}_{\overline{H}}$).

\subsection{Biased and Unbiased Worlds}
\label{subsec:depworlds}

Recall that in the non-stationary setting, not all $x \in X$ were mutually disjoint---which means their induced MDPs had dependent transitions. We resolved this issue by partitioning $X$ into $N$-sized \textit{batches} of mutually-disjoint worlds. In the stationary setting, we could encounter an issue of dependence even \textit{within} a single world $x \in X$. Consider the world $x = 4416823295$ from Example~\ref{ex:D}: this world is constrained to
set both $T_{x}(s_{0}, a_{0}, 0, s_{0})$ and 
$T_{x}(s_{0}, a_{0}, 1, s_{0})$ based on the the \textit{same} sample, namely the $4^{\text{th}}$ one collected for $(s_{0}, a_{0})$. Consequently,  
$T_{x}(s_{0}, a_{0}, 0, s_{0})$ is \textit{dependent} on $V^{\pi}_{x}(s_{0}, 1)$ for any policy $\pi$ that takes $a_{0}$ from $s_{0}$ at time step $1$. We can no longer claim $\mathbb{E}[V^{\pi}_{x}] = V^{\pi}_{M_{\overline{H}}}$ (like we did while analysing CEM-NS, in the proof of Lemma~\ref{lem:ns-exp}). 

To proceed, we partition $X$ into sets $X_{\text{biased}}$ and $X_{\text{unbiased}}$. The set $X_{\text{biased}}$ contains all worlds $x \in X$ for which there exist $(s, a, t, t^{\prime}) \in S \times A \times [\overline{H}] \times [\overline{H}]$, $t \neq t^{\prime}$, such that $x(s, a, t) = x(s, a, t^{\prime})$. Such worlds induce MDPs that provide possibly biased value estimates. The complementary set $X_{\text{unbiased}} \eqdef X \setminus X_{\text{biased}}$ 
%, which we can handle as in Section~\ref{sec:ssw}. In contrast with worlds in $X_{\text{biased}}$, every world $x \in X_{\text{unbiased}}$ indeed 
contains worlds that do
provide an unbiased estimate of the value function of each policy $\pi: S \times [\overline{H}] \to A$ on $M_{\overline{H}}$.
\begin{lemma}[Worlds in $X_{\text{unbiased}}$ provide unbiased estimates]
\label{lem:s-exp}
For $x \in X_{\textup{unbiased}}$, $\pi: S \times [\overline{H}] \to A$, $(s, t) \in S \times [\overline{H}]$,
$$\mathbb{E}[V^{\pi}_{x}(s, t)] = V^{\pi}_{M_{\overline{H}}}(s, t).$$
\end{lemma}
The proof is identical to the one of Lemma~\ref{lem:ns-exp}, relying on the independence of random variables $T_{x}(s, \pi(s), t, s^{\prime})$ and $V^{\pi}_{x}(s^{\prime}, t + 1)$ for $x \in X_{\text{unbiased}}$. %($V^{\pi}_{x}(s^{\prime}, t + 1)$ depends on the time step due to the finite horizon $\overline{H}$, even if $\pi$ is stationary.)

Without any useful handle on worlds $x \in X_{\text{biased}}$, our strategy is to show that the size of 
$X_{\text{biased}}$ as a fraction of $|X|$ vanishes with $N$, implying that $V^{\pi}_{X_{\text{biased}}}$ influences $V^{\pi}_{X}$ only marginally when $N$ is sufficiently large. The following lemma is proven in Appendix~\ref{app:proofoflemmabiasedsize}.

%\newpage
\begin{lemma}[Error from biased worlds vanishes with $N$]
\label{lem:XunbiasedclosetoX}
For $\pi: S \times [\overline{H}] \to A$, $(s, t) \in S \times [\overline{H}]$: 
$$\left|V^{\pi}_{X}(s, t) - V^{\pi}_{X_{\textup{unbiased}}}(s, t)\right| \leq \frac{|S||A|\overline{H} (\overline{H} - 1) V_{\max}}{N}.$$
\end{lemma}

Finally, just as we grouped $x \in X$ into mutually-disjoint batches in Section~\ref{sec:ssw}, we do the same for $x \in X_{\text{unbiased}}$ in the stationary setting. We do not consider worlds in $X_{\text{biased}}$ for this grouping. Recall that worlds $x$ and $x^{\prime}$ are disjoint if and only if $x(s, a, t) \neq x^{\prime}(s, a, t)$ for all $(s, a, t) \in S \times A \times [\overline{H}]$. Assume for simplicity that $N$ is a multiple of $\overline{H}$, and define $N^{\prime} = N / \overline{H}$. 
%The $N$ samples for each $(s, a) \in S \times [A]$ in $D$ get divided into roughly $N^{\prime}$ samples for each $(s, a, t)$, $t [\overline{H}]$.%gets $N^{\prime}$ samples on average in a world from $X_{\text{unbiased}}$. 
Calculations provided in Appendix~\ref{app:batch-counting} show that (1) $|X_{\text{unbiased}}| = \frac{N!^{|S||A|}}{(N - \overline{H})!^{|S||A|}}$, (2) the set of all batches $B$
(each batch containing $N^{\prime}$ mutually-disjoint worlds $x \in X_{\text{unbiased}}$)
is of size $\frac{N!^{|S||A|}}{N^{\prime}!}$, and (3) the set of all batches $B_{x}$ that contain any particular world $x \in X_{\text{unbiased}}$ is of size $\frac{(N - \overline{H})!^{|S||A|}}{(N^{\prime} - 1)!}$. Substituting into a working similar to \eqref{eqn:XasBsum}, we observe
\begin{align}
V^{\pi}_{X_{\text{unbiased}}}(s, t) = \frac{1}{|B|} \sum_{b \in B} V^{\pi}_{b}(s, t)
\label{eq:XAsBsum-s}
\end{align}
for $\pi: S \times [\overline{H}] \to A$, $(s, t) \in S \times [\overline{H}]$, which facilitates the use of Lemma~\ref{lem:dephoeffding} on $V^{\pi}_{X_{\text{unbiased}}}$.

%Expressing $V^{\pi}_{X_{\text{indep}}}$ as a convex combination of $|B|$ random variables $V^{\pi}_{b}$, $b \in B$, we again apply Hoeffding's inequality on sums of dependent random variables (exactly as in our proof of Lemma~\ref{lem:Xdeviation}) to obtain: for $\pi: S \times [H] \to A$, $(s, t) \in S \times [H]$, $\epsilon \in (0, V_{\max})$,
%\begin{align}
%\mathbb{P}\{ V^{\pi}(s, t) -
%V^{\pi}_{X_{\text{indep}}}(s, t) > \epsilon\} &\leq \exp\left(- \frac{2N\epsilon^{2}}{(V_{\max})^{2}}\right); \nonumber \\
%\mathbb{P}\{V^{\pi}_{X_{\text{indep}}}(s, t) -
%V^{\pi}(s, t) > \epsilon\} &\leq  \exp\left(- \frac{2N\epsilon^{2}}{(V_{\max})^{2}}\right).
%\label{eq:Xindepdeviation}
%\end{align}

%\subsection{CE-S: theory and practice}
%\label{subsec:scub}
%\shivaram{Call this section m o s c}

We have all the elements ready for an upper bound on the sample complexity of CEM-S.

%by combining \eqref{eq:depratio} and \eqref{eq:Xindepdeviation}.

\begin{theorem}[Sample complexity of CEM-S]
\label{thm:stat}
The CEM-S algorithm provides the relevant PAC guarantee for stationary MDP $M$ with parameters $\epsilon \in (0, V_{
max})$, $\delta \in (0, 1)$ if run with
$$N = \max\left( \left\lceil \frac{32(V_{\max})^{2}}{\epsilon^{2}} \ln \frac{|S||A|^{|S|}}{\delta}\right\rceil, \left\lceil \frac{8 |S||A| (\overline{H} - 1)V_{\max}}{\epsilon} \right\rceil \right) \overline{H}.$$
\end{theorem}
The proof (given in detail in Appendix~\ref{app:proofofcemsupperbound}) follows the same core structure as of the non-stationary case in Theorem~\ref{thm:non-stat}, but requires additional steps to account for the truncated horizon $\overline{H}$ and the partition of $X$ into sets $X_{\text{biased}}$ and $X_{\text{unbiased}}$. 
%From Theorem~\ref{thm:stat} w
We infer that the sample complexity of CEM-S is %at most
%{\small 
$$O\left(\frac{|S||A|}{(1 - \gamma)^{3}\epsilon^{2}} \left(\log \frac{1}{(1 - \gamma)\epsilon}\right)\left(\log \frac{1}{\delta} + |S||A| \epsilon \log \frac{1}{1 - \gamma} \right)\right).$$
%}
Unlike existing upper bounds~\citep{Azar2013GenModel,Agarwal2020GenModel} that only hold for restricted ranges of $\epsilon$, this bound applies to the entire range of problem parameters.
%Ignoring logarithmic factors in $\frac{1}{\epsilon}$ and $\frac{1}{1 - \gamma}$,
Observe that the coefficient of $\log(\frac{1}{\delta})$ is $\tilde{O}\left(\frac{|S||A|}{(1 - \gamma)^{3} \epsilon^{2}}\right)$. Thus, we have the novel result that for arbitrary, fixed values of $|S|$, $|A|$, $\epsilon$, and $\gamma$, CEM-S is optimal up to logarithmic factors as $\delta \to 0$. The notion of optimality in the limit as $\delta \to 0$ has also been applied in other PAC learning contexts~\citep{pmlr-v49-garivier16a}.

%and matches
%With $\widehat{\pi}$ in hand, it takes fewer samples (asymptotically) to estimate $Q^{\star}$. Hence, Theorem~\ref{thm:stat}delivers the novel result that \textbf{CE-S} meets the lower bound of \cite{Azar2013GenModel} to within a logarithmic factor (in $\frac{1}{\epsilon}$ and $\frac{1}{1 - \gamma}$) for $|S||A|\epsilon\log\frac{1}{1 - \gamma} \ll \log\frac{1}{\delta}$. 

\begin{comment}
Empirical evidence suggests that \textbf{CE-S} might be more efficient than the current analysis suggests. On standard benchmark MDPs such as ``River Swim''~\citep{Strehl2008MBIE} we find the empirical mistake probability of \textbf{CE-S} after any fixed $N$ samples to be lower than that of \textbf{CE-NS}. Figure~\ref{fig:ces-cens} considers an example MDP with a single stochastic transition (parameterised by $p \in [0, 1]$). The plots in figures~\ref{fig:p-0.01} and \ref{fig:p-0.001} plot the mistake probability for two settings of $p$. In the more typical case ($p = 0.01$), observe that \textbf{CE-S} uniformly dominates \textbf{CE-NS}. For $p = 0.001$, the opposite trend is seen for small values of $N$, yet \textbf{CE-S} comfortably dominates beyond $N = 2^{8}$. The termination condition of \textbf{CE-NS}, even with $\delta = 1$, is only for $N > 2^{23}$. These results motivate a continued quest for a proof of the minimax-optimality of \textbf{CE-S} for the full range of input parameters. 

\input{figure_101}

\end{comment}

\section{Sample-Complexity Lower Bound For Finite-Horizon MDPs}
\label{sec:lowerbound}
In this section, we furnish a lower bound on the sample complexity required for any PAC algorithm on finite-horizon MDPs. This bound, shown by constructing a family of stationary MDPs, applies to both stationary and to non-stationary MDPs. The basic structure is taken from \citet{Azar2013GenModel}, who provide a similar lower bound for infinite-horizon MDPs with discounting. The main change required is to substitute terms depending on discount factor $\gamma$ with terms depending on horizon $H$. We also note that the lower bound of \citet{Azar2013GenModel} applies only to a restricted class of algorithms that make ``independent'' predictions for ``independent'' state-action pairs~\cite[see Lemma 18]{Azar2013GenModel}. To obtain a more general result, we also borrow from the proof structure used by \citet{Mannor+Tsitsiklis:2004} for best-arm identification in stochastic multi-armed bandits. 

At each step, algorithm $\mathcal{L}$ has a history of (state, action, time step) triples that have already been sampled, along with their observed outcomes (next states). The algorithm must either (1) stop and publish $\hat{Q}(s, a, 0)$ for each $(s, a) \in S \times A$, or (2) specify a probability distribution over all $(s, a, t) \in S \times A \times [H]$. If the latter, an $(s, a, t)$ triple is sampled, its outcome recorded, and the process moves to the next step. For $\epsilon > 0$, the output of $\mathcal{L}$ is $\epsilon$-correct if for all $(s, a) \in S \times A$, $|Q^{\star}(s, a, 0) - \hat{Q}(s, a, 0)| \leq \epsilon$. In turn, for $\delta > 0$, $\mathcal{L}$ is an $(\epsilon, \delta)$-PAC algorithm if on each input MDP, the probability that $\mathcal{L}$ stops and returns an $\epsilon$-correct output is at least $1 - \delta$. 

Observe that we provide a lower-bound for accurately estimating $Q$-values; this is mainly for being consistent with \citet{Azar2013GenModel}. It is easily shown that the same lower bound holds (up to a constant factor) for estimating an $\epsilon$-optimal policy with probability $1 - \delta$. In our working, we use $Q^{\star}(s, a)$ and $\hat{Q}(s, a)$ to denote $Q^{\star}(s, a, 0)$ and $\hat{Q}(s, a, 0)$, respectively. Below is our formal statement.

\begin{theorem}
\label{thm:lower_bound_stationary}
Fix set of states $S$, set of actions $A$ arbitrarily, and let horizon $H > 200$. There exist an MDP $(S, A, T, R, H, \gamma = 1)$, constants $c_1>0, c_2>0$ 
%and an episodic MDP $M$ with $k$ state-action pairs and horizon $H$ 
such that for all $\epsilon \in (0, 1)$, $\delta \in (0, 0.5)$, any $(\epsilon, \delta)-PAC$ algorithm $\mathcal{L}$
%that maps the data of samples $D$ to an output $\hat{Q}$ must have 
has an expected sample complexity of at least
    \begin{align*}
    %E[\tau] \geq &
    \frac{c_1 |S||A| H^3}{\epsilon^2} \ln\left( \frac{c_2}{\delta} \right)
    \end{align*}
on this MDP.
\end{theorem}
The full proof of this theorem is provided in Appendix~\ref{app:lowerboundproof}. The proof relies on a purposefully-designed family of MDPs constructed by \citet{Azar2013GenModel}. These authors consider two specific MDPs, $M_{0}$ and $M_{1}$, whose $Q$-functions are more than $2\epsilon$ apart, and hence any estimate cannot simultaneously be $\epsilon$-correct for both MDPs. On the other hand, since the MDPs are sufficiently close, an agent cannot distinguish them based on observed samples alone unless the sample size is sufficiently large. Consequently, unless a sufficient number of samples are observed, an algorithm must necessarily be non-$(\epsilon, \delta)$-PAC.

Our proof follows this same structure, except (1) we use a finite horizon $H$ for the MDP family, and (2) we consider $\frac{k}{3} + 1$ MDPs where $k$ is the number of state-action pairs. The latter adaptation lets us extend the lower bound to arbitrary algorithms. In our proof, there is a base MDP $M_{0}$ and for each state-action pair $i$ that represents a choice, an MDP $M_{i}$ that differs from $M_{0}$ only at this state-action pair.
%To derive the lower bound, we construct the aforementioned family of MDPs that share similar data observation probabilities but have significantly different Q-values. By comparing any MDP $M_i$ in this family to $M_0$, which differs only at the $i$-th state-action pair, w
We lower-bound the ratio of the likelihood of observing the same data from these MDPs in terms of the number of samples. If the sample complexity is ``small'', the likelihood ratio bound allows us to show that if a certain event has a high probability under MDP $M_0$, it also has a high probability under $M_i$. By choosing this event to be the event that $|\hat{Q} - Q_0^*| < \epsilon$,
%where $\hat{Q}$ is the estimated action-value function and $Q_0$ is the true optimal action-value function of $M_0$, 
we argue that if an algorithm gives an $\epsilon$-correct answer with probability at least $1 - \delta$ on $M_{0}$, it gives this same (but now \textit{not} $\epsilon$-correct) answer on $M_{i}$ with probability at least $1 - \delta$. Hence an algorithm cannot be PAC unless the sample complexity is ``large''---as quantified in Theorem~\ref{thm:lower_bound_stationary}.

\section{Conclusion}
\label{sec:conclusion}

In this paper, we bring to light a surprising connection between the well known certainty-equivalence method (CEM) for PAC RL, and the trajectory tree method (TTM) for decision-time planning. We show that CEM implicitly computes a policy that simultaneously optimises over all possible ``batches'' of worlds, whereas TTM explicitly sets up a single batch of trajectory trees (functionally akin to worlds) to compute its policy. Noticing this connection, we establish upper bounds for CEM using only Hoeffding's inequality, yet which improve upon current bounds in the regime of small $\delta$. Our results are especially significant in the finite-horizon (non-stationary) setting, where in spite of making a weaker assumption on the rewards, we show the minimax-optimality of CEM in the small-$\delta$ regime.
%We make a weaker assumption on the eThrough a lower bound, we show the optimality of CEM in the non-stationary (finite horizon) regime.
%dditionally show optimality in this regime for 

Our new perspective sets up several possible directions for future work, including (1) the derivation of instance-specific upper bounds for sequential PAC RL algorithms, and (2) generalising the idea to formalisms that use function approximation.
%There also remain technical challenges with our current results, such as  tightening the $\delta$-independent terms in our upper bounds.
%, and (2) furnishing a \textit{lower bound} for PAC RL in the non-stationary setting, to compare with our current upper bound for \textbf{CEM-NS}.
Finally, (3) it would be worth investigating the applicability of our analytical framework to related on-line learning problems,  such as exploration in continuing tasks without ``reset'' access~\citep{Brafman2002RMAX,Strehl2008MBIE}, the episodic off-policy setting~\citep{Yin2021FiniteH_AISTATS}, and regret minimisation
~\citep{Auer2006UCRL,DannBrunskill2017}.

\bibliographystyle{apalike}
\bibliography{main}

%\end{document}

\onecolumn
\appendix
\section{Proofs of Lemma~\ref{lem:XequalsbarmMNS} and Lemma~\ref{lem:XequalsbarmMS}}
\label{app:proofsofXequalsM}

For proving Lemma~\ref{lem:XequalsbarmMNS}, we have to show that for $\pi: S \times [H] \to A$, $(s, t) \in S \times [H]$: $V^{\pi}_{X}(s, t) = V^{\pi}_{\widehat{M}}(s, t).$ We use induction on $t$. For the base case of $t = H - 1$, notice that $V^{\pi}_{X}(s, H - 1) = R(s, \pi(s, H - 1), H - 1) = V^{\pi}_{\widehat{M}}(s, H - 1)$. Suppose that the claim is true for $t + 1$, where $t \in \{0, 1, \dots, H - 2\}$. We have
\begin{align*}
V^{\pi}_{X}(s, t) &= \frac{1}{|X|} \sum_{x \in X} V^{\pi}_{x}(s, t) \\
&= \frac{1}{|X|} \sum_{x \in X} \left( R(s, \pi(s, t), t) + \gamma \sum_{s^{\prime} \in S} T_{x}(s, \pi(s, t), t, s^{\prime} ) V^{\pi}_{x}(s^{\prime}, t + 1)\right) .
\end{align*}
Recall that $x$ is a string of length $|S||A|H$ over $[N]$, organised as $H$ segments, with each $|S||A|$-length segment specifying the index of the sample from $[N]$ for each state-action pair. Let
\begin{itemize}
\item $x_{-}$ be $x$'s prefix of length $|S||A|t$ (empty if $t = 0$);
\item $x_{\circ}$ be the subsequent ``middle portion'' of length $|S||A|$, and \item $x_{+}$ be the suffix of length $|S||A|(H - t - 1)$.
\end{itemize}

In other words, $x = x_{-} x_{\circ} x_{+}$ for some $x_{-} \in X_{-} \eqdef [N]^{|S||A|t}$, $x_{\circ} \in X_{\circ} \eqdef [N]^{|S||A|}$, $x_{+} \in X_{+} \eqdef [N]^{|S||A|(H - t - 1)}$. We continue based on this decomposition:

\begin{align*}
V^{\pi}_{X}(s, t) &= R(s, \pi(s, t), t) 
+ \frac{\gamma}{|X|} \sum_{s^{\prime} \in S} \sum_{x_{-} \in X_{-}} \sum_{x_{\circ} \in X_{\circ}} \sum_{x_{+} \in X_{+}} T_{x_{-}x_{\circ}x_{+}}(s, \pi(s, t), t, s^{\prime} ) V^{\pi}_{x_{-}x_{\circ}x_{+}}(s^{\prime}, t + 1) .
\end{align*}
Now, $T_{x_{-}x_{\circ}x_{+}}(s, \pi(s, t), t, s^{\prime})$ does not depend on $x_{-}$ or $x_{+}$, so we can simply denote it $T_{x_{\circ}}(s, \pi(s, t), t, s^{\prime} )$. 

Similarly, $V^{\pi}_{x_{-}x_{\circ}x_{+}}(s^{\prime}, t + 1)$ does not depend on $x_{-}$ or $x_{\circ}$, and so we may denote it $V^{\pi}_{x_{+}}(s^{\prime}, t + 1)$. With this  observation, we see that
\begin{align*}
V^{\pi}_{X}(s, t) &= R(s, \pi(s, t), t) 
+ \\
&\frac{\gamma}{|X|} \sum_{s^{\prime} \in S} |X_{-}| \left(\sum_{x_{\circ} \in X_{\circ}}
T_{x_{\circ}}(s, \pi(s, t), t, s^{\prime} ) \right) \left( \sum_{x_{+} \in X_{+}}    V^{\pi}_{x_{+}}(s^{\prime}, t + 1) \right).
\end{align*}
The term $\sum_{x_{\circ} \in X_{\circ}}
T_{x_{\circ}}(s, \pi(s, t), t, s^{\prime} )$ is seen to be $$N^{|S||A| - 1} \text{\text{count}}(s, \pi(s, t), t, s^{\prime}) = N^{|S||A|} \widehat{T}(s, \pi(s, t), t, s^{\prime} );$$ 
the $N^{|S||A| - 1}$ factor comes from other state-action pairs for time step $t$ taking every possible index from $[N]$ in the set of worlds. Since $V^{\pi}_{x_{+}}(s^{\prime}, t + 1)$ does not depend on $x_{-}$ and $x_{\circ}$, but $x_{-}$ and $x_{\circ}$ take all possible values for each $x_{+} \in X_{+}$, we substitute  $\sum_{x_{+} \in X_{+}} V^{\pi}_{x_{+}}(s^{\prime}, t + 1)$ with $$\frac{|X_{+}|}{|X|} \sum_{x^{\prime} \in X}    
V^{\pi}_{x'}(s^{\prime}, t + 1) = |X_{+}|V^{\pi}_{X}(s^{\prime}, t + 1)
= |X_{+}|V^{\pi}_{\widehat{M}}(s^{\prime}, t + 1),$$ where the last step applies the induction hypothesis. In aggregate, we now have
\begin{align*}
V^{\pi}_{X}(s, t) &= R(s, \pi(s, t), t) 
+ \frac{\gamma}{|X|} \sum_{s^{\prime} \in S} |X_{-}| N^{|S||A|} \widehat{T}(s, \pi(s, t), t, s^{\prime}) |X_{+}| V^{\pi}_{\widehat{M}}(s^{\prime}, t + 1)\\
&= R(s, \pi(s, t), t) + \gamma \sum_{s^{\prime} \in S} \widehat{T}(s, \pi(s, t), t, s^{\prime}) V^{\pi}_{\widehat{M}}(s^{\prime}, t + 1) = V^{\pi}_{\widehat{M}}(s, t),
\end{align*}
which completes the proof of Lemma~\ref{lem:XequalsbarmMNS}.

The proof of Lemma~\ref{lem:XequalsbarmMS} is identical: we only have to ignore the time-independence of $R$ and $T$, and take $\overline{H}$ as horizon instead of $H$.

%\noindent \textbf{Significance.} Our new framework of analysis establishes the following results. (1) CE-NS improves the best-known upper bound for non-stationary MDPs by a factor of $H$. (2) CE-NS matches the best-known upper bound for stationary MDPs, with a simple conceptual interpretation. (3)  CE-S is minimax-optimal for stationary MDPs as $\delta \to 0$.

\section{Counting with Batches}
\label{app:batch-counting}

We compute the sizes of the set of all batches $B$, and the set of all batches containing some fixed world $x$. For the non-stationary setting, the set of worlds relevant to this exercise is the universe $X$, whereas for the stationary setting, counting is only done on $X_{\text{unbiased}}$.

\subsection{Non-Stationary Setting}

A batch contains $N$ worlds such that for any pair of distinct worlds $x, x^{\prime}$ in $X$, there is no $(s, a, t) \in S \times A \times [H]$ such that $x(s, a, t) = x^{\prime}(s, a, t)$. To count the total number of such batches possible, it is helpful to visualise a table with $N$ rows and $|S||A|H$ columns. In how many ways can we fill up the cells with numbers from $[N]$ such that no two rows share a common element in any column?

\begin{center}
\begin{tabular}{|c|c|c|c|c|c|}
\hline
& ~~$1$~~ & ~~$2$~~ & ~~$3$~~ & ~$\dots$~ & ~$|S||A|H$~ \\ \hline
~$1$~ & & & & &\\ \hline
~$2$~ & & & & &\\ \hline
~$3$~ & & & & &\\ \hline
~$\vdots$~ & & & & & \\ \hline
~$N$~ & & & & & \\ \hline
\end{tabular}
\end{center}

We proceed row by row. The first row can be filled up in $N^{|S||A|H}$ possible ways. With the first row filled, the second row can be filled in $(N - 1)^{|S||A|H}$ possible ways. We proceed in this manner, until entries for the last row are fixed by those from the preceding $N - 1$ rows. Hence, the number of possible tables we could have created is $$N^{|S||A|H} \times (N - 1)^{|S||A|H} \times (N - 2)^{|S||A|H} \times 1^{|S||A|H} = N!^{|S||A|H}.$$ A batch is a \textit{set} rather than a sequence---so any two tables that contain the same contents in each row, even if the rows are permuted, fall in the same equivalence class of a batch. Since there are $N!$ possible ways to permute the rows of the table, we observe that the number of unique batches is $|B| = \frac{N!^{|S||A|H}}{N!} = N!^{|S||A|H - 1}.$

If a particular world (here row) $x$ is fixed, the same reasoning implies the number of batches containing $x$ is $|B_{x}| = (N - 1)!^{|S||A|H - 1}$.

\subsection{Stationary Setting}

Recall that world $x \in X$ is \textit{biased} if there exist $(s, a, t, t^{\prime}) \in S \times A \times [\overline{H}] \times [\overline{H}]$, $t \neq t^{\prime}$, such that $x(s, a, t) = x(s, a, t^{\prime})$. Hence, in an unbiased world $x \in X_{\text{unbiased}}$, for each $(s, a) \in S \times A$, we must pick distinct samples for each $t \in [\overline{H}]$. The number of ways to select an $\overline{H}$-sized permutation from $[N]$ is $N(N - 1)(N - 2)\dots(N - \overline{H} + 1) = \frac{N!}{(N - \overline{H})!}.$Since we gather such a permutation for each $(s, a) \in S \times A$, the size of $X_{\text{unbiased}}$ is $\left(\frac{N!}{(N - \overline{H})!}\right)^{|S||A|}.$

In the stationary setting, we consider batches of size  $N^{\prime} = N/\overline{H}$, which are mutually-disjoint sets of worlds from $X_{\text{unbiased}}$. In a world $x \in X_{\text{unbiased}}$, no index from $[N]$ can repeat within the sub-table for each state-action pair. Other than for that constraint, we can calculate the size of $B$ as before. 

To fill up the first row of the table, we must pick some $\overline{H}$ indices from $[N]$ for each state-action pair. The number of ways we can do this is 
$$\left(N(N - 1)(N - 2)\dots(N - \overline{H} + 1)\right)^{|S||A|}.$$ None of the $\overline{H}$ indices used for any state-action pair in the first row can be used subsequently. Correspondingly, the number of ways to fill up the second row becomes
$$\left((N - \overline{H})(N - \overline{H} - 1)(N - \overline{H} - 2)\dots(N - 2 \overline{H} + 1)\right)^{|S||A|}.$$ Proceeding similarly, the number of ways to fill up the last  ($N^{\prime}$-th) row is $$\left((N - k\overline{H})(N - k\overline{H} - 1)(N - k\overline{H} - 2)\dots(N - (k +1) \overline{H} + 1)\right)^{|S||A|}$$ for $k = N^{\prime} - 1$. Taking a product over rows, we observe that the total number of ways to fill it up is $N!^{|S||A|}$. Once again, we must account for permutations of the rows, which are $N^{\prime}$ in number. Hence, the number of batches $|B| = \frac{N!^{|S||A|}}{N^{\prime}!}.$

If we fix a particular row with world $x$, it leaves $N^{\prime} - 1$ rows to fill out, but only $N - \overline{H}$ indices available for each state-action pair. Repeating the same counting argument, we get $|B_{x}| = \frac{(N - \overline{H})!^{|S||A|}}{(N^{\prime} - 1)!}.$

\section{Proof of Lemma~\ref{lem:dephoeffding}}
\label{app:dephoeffdingproof}

The lemma and the proof are from \citet[see Section 5]{Hoeffding:1963}. The proof is rephrased here for the reader's convenience. For arbitrary $h > 0$, we have
\begin{align}
\mathbb{P}\{U \geq \mathbb{E}[U]  + \gamma\}
&= 
\mathbb{P}\{h(U - \mathbb{E}[U])  \geq h \gamma\} \nonumber \\
&= 
\mathbb{P}\{\exp(h (U - \mathbb{E}[U] )) \geq \exp(h\gamma))\}.
\label{eqn:dh-1}
\end{align}
Applying Markov's Inequality to the non-negative random variable $\exp(h (U - \mathbb{E}[U] ))$, we observe 
\begin{align}
\mathbb{P}\{\exp(h (U - \mathbb{E}[U] )) \geq \exp(h\gamma))\}
&\leq \exp(-h \gamma) \mathbb{E}[\exp(h (U - \mathbb{E}[U] ))] \nonumber \nonumber \\
&= \exp(-h \gamma)
\mathbb{E}\left[\exp\left(h \sum_{i = 1}^{\ell} p_{i}(U_{i} - \mathbb{E}[U_{i}])\right)\right].
\label{eqn:dh-2}
\end{align}
In turn, since the exponential function is convex, by Jensen's inequality,
\begin{align}
\exp\left(h \sum_{i = 1}^{\ell} p_{i}(U_{i} - \mathbb{E}[U_{i}])\right) \leq \sum_{i = 1}^{\ell} p_{i} \exp (h (U_{i} - \mathbb{E}[U_{i}])).
\label{eqn:dh-3}
\end{align}
Combining \eqref{eqn:dh-1}, \eqref{eqn:dh-2}, and \eqref{eqn:dh-3}, we have
\begin{align}
\mathbb{P}\{U \geq \mathbb{E}[U] + \gamma\}
&\leq \exp(-h \gamma )
\sum_{i = 1}^{\ell} p_{i} \mathbb{E}[\exp (h (U_{i} - \mathbb{E}[U_{i}]))].
\label{eqn:dh-4}
\end{align}
The application of Jensen's inequality to obtain \eqref{eqn:dh-3} above is not needed for the common application of Hoeffding's Inequality to averages of \textit{independent} random variables, which is the case for us if $\ell = 1$. For the general case, \eqref{eqn:dh-3} enables us to upper-bound the deviation probability by a convex combination of expectations, one corresponding to each $i \in \{1, 2, \dots, \ell\}$. Since each $U_{i}$ is indeed an average of independent random variables, each expectation can be upper-bounded exactly as in the proof of the common variant~\cite[see Theorem 2]{Hoeffding:1963}. For $i \in \{1, 2, \dots, \ell\}$,
\begin{align}
\mathbb{E}[\exp (h (U_{i} - \mathbb{E}[U_{i}]))]
&= \mathbb{E}\left[\exp \left(\frac{h}{m} \sum_{j = 1}^{m} (U_{i, j} - \mathbb{E}[U_{i, j}]\right)\right]\nonumber \\
&= \mathbb{E}\left[ \prod_{j = 1}^{m}  \exp\left(\frac{h}{m} (U_{i, j} - \mathbb{E}[U_{i, j}])\right)\right]\nonumber \\
&= \prod_{j = 1}^{m}  \mathbb{E}\left[ \exp\left(\frac{h}{m} (U_{i, j} - \mathbb{E}[U_{i, j}])\right)\right],\label{eqn:dh-5}
\end{align}
where the last step follows from the independence of $U_{i, j}$ and $U_{i, j^{\prime}}$ for 
$j, j^{\prime} \in \{1, 2, \dots, m\}$,
$j \neq j^{\prime}$. At this point, we apply a mathematical fact that is sometimes called Hoeffding's lemma~\cite[see Section 4 for proof]{Hoeffding:1963}. Noting that $U_{i, j}$ for $i \in \{1, 2, \dots, \ell\}, j \in \{1, 2, \dots, m\}$ is bounded in $[\alpha, \beta]$, we have
\begin{align}
\mathbb{E}\left[ \exp\left(\frac{h}{m} (U_{i, j} - \mathbb{E}[U_{i, j}])\right)\right] &\leq \exp\left(\frac{h^{2}(\beta - \alpha)^{2}}{8 m^{2}}\right).
\label{eqn:dh-6}
\end{align}
Combining \eqref{eqn:dh-4}, \eqref{eqn:dh-5}, and \eqref{eqn:dh-6}, we obtain
\begin{align}
\mathbb{P}\{U \geq \mathbb{E}[U] + \gamma\}
&\leq \exp(-h \gamma ) \sum_{i = 1}^{\ell} p_{i} \prod_{j = 1}^{m} \exp\left(\frac{h^{2}(\beta - \alpha)^{2}}{8 m^{2}}\right) \nonumber \\
&= \sum_{i = 1}^{\ell} p_{i} \exp\left(-h \gamma + \frac{h^{2}(\beta - \alpha)^{2}}{8 m}\right).\label{eqn:dh-7}
\end{align}
Substituting the particular choice of
$h = \frac{4 \gamma m}{(\beta - \alpha)^{2}}$ in \eqref{eqn:dh-7}, we conclude
\begin{align}
\mathbb{P}\{U \geq \mathbb{E}[U] + \gamma\}
&\leq \sum_{i = 1}^{\ell} p_{i}
\exp\left(\frac{-2 m \gamma^{2}}{(\beta - \alpha)^{2}} \right) \nonumber \\
&= \exp\left(\frac{-2 m \gamma^{2}}{(\beta - \alpha)^{2}} \right).
\label{eqn:dh-8}
\end{align}
For $i \in \{1, 2, \dots, \ell\}$ and $j, j^{\prime} \in \{1, 2, \dots, m\}$, $j \neq j^{\prime}$, our precondition that $U_{i, j}$ and $U_{i, j^{\prime}}$ are independent implies that $-U_{i, j}$ and $-U_{i, j^{\prime}}$ are also independent. Also note that $-U_{i, j}$ is supported on $[-\beta, -\alpha]$. We can therefore present the same proof as above to obtain $$\mathbb{P}\{-U \geq \mathbb{E}[-U] + \gamma\}
\leq \exp\left(\frac{-2 m \gamma^{2}}{(-\alpha + \beta)^{2}} \right),$$ or equivalently, 
\begin{align}
\mathbb{P}\{U \leq \mathbb{E}[U] - \gamma\} \leq \exp\left(\frac{-2 m \gamma^{2}}{(\beta -\alpha)^{2}} \right).
\label{eqn:dh-9}
\end{align}
\eqref{eqn:dh-8} and \eqref{eqn:dh-9} complete the proof of the lemma.

\section{Proof of Lemma~\ref{lem:XunbiasedclosetoX}}
\label{app:proofoflemmabiasedsize}

We need to show that for $\pi: S \times [\overline{H}] \to A$, $(s, t) \in S \times [\overline{H}]$:
$|V^{\pi}_{X}(s, t) - V^{\pi}_{X_{\text{unbiased}}}(s, t)| \leq \frac{|S||A|\overline{H} (\overline{H} - 1) V_{\max}}{N}.$

First we decompose $V^{\pi}_{X}(s, t)$ for $\pi \in S \times [\overline{H}] \to A$ and $(s, t) \in S \times [\overline{H}]$ as
\begin{align*}
V^{\pi}_{X}(s, t) 
&= \frac{1}{|X|} \sum_{x \in X} V^{\pi}_{x}(s, t) 
=
\frac{|X_{\text{unbiased}}|}{|X|}
V^{\pi}_{X_{\text{unbiased}}}(s, t) 
+
\frac{|X_{\text{biased}}|}{|X|}
V^{\pi}_{X_{\text{biased}}}(s, t) \\
& = V^{\pi}_{X_{\text{unbiased}}}(s, t) 
+
\frac{|X_{\text{biased}}|}{|X|}
\left(V^{\pi}_{X_{\text{biased}}}(s, t) - V^{\pi}_{X_{\text{unbiased}}}(s, t) \right),
\end{align*}
which implies
$$|V^{\pi}_{X}(s, t) - V^{\pi}_{X_{\text{unbiased}}}(s, t)| =  
\frac{|X_{\text{biased}}|}{|X|}
|V^{\pi}_{X_{\text{biased}}}(s, t) - V^{\pi}_{X_{\text{unbiased}}}(s, t) |.$$ $|V^{\pi}_{X_{\text{biased}}}(s, t) - V^{\pi}_{X_{\text{unbiased}}}(s, t) |$ is at most $V_{\text{max}}$. The proof is completed by observing     
\begin{align}
\frac{|X_{\text{biased}}|}{|X|}
&= 1 - \frac{|X_{\text{unbiased}}|}{|X|} = 1 - \frac{(N (N - 1)\dots(N - \overline{H} + 1))^{|S||A|}}{N^{|S||A|\overline{H}}}\nonumber\\
&\leq 1 - \left( 1 - \frac{\overline{H} - 1}{N}\right)^{|S||A|\overline{H}}
\leq \frac{|S||A|\overline{H}(\overline{H} - 1)}{N}.\nonumber
%\label{eq:depratio}
\end{align}

\section{Proof of Theorem~\ref{thm:stat}}
\label{app:proofofcemsupperbound}

CEM-S returns $\widehat{\pi}: S \to A$, which is an optimal policy for $\widehat{M}$. If 
$\widehat{\pi}$ is not $\epsilon$-optimal (for $M$), then either (i) $\widehat{M}$  has under-estimated $\pi^{\star}$ by at least $\frac{\epsilon}{2}$ for some state $s \in S$, or (ii) $\widehat{M}$ has over-estimated some other policy $\pi: S \to A$ by at least $\frac{\epsilon}{2}$ for some state $s \in S$. We upper-bound the probability of these events, in turn.

Suppose for state $s \in S$, we have $V^{\pi^{\star}}_{\widehat{M}}(s) \leq V^{\star}(s) - \frac{\epsilon}{2}$. Proposition~\ref{prop:trunc} then implies $V^{\pi^{\star}}_{\widehat{M}_{\overline{H}}}(s, 0) \leq V^{\star}(s) - \frac{\epsilon}{2}$. Again from  Proposition~\ref{prop:trunc}, $V^{\pi^{\star}}_{\widehat{M}_{\overline{H}}}(s, 0) \leq V^{\star}_{M_{\overline{H}}}(s. 0) - \frac{\epsilon}{4}$. Equivalently, from Lemma~\ref{lem:XequalsbarmMS}, we get  $V^{\pi^{\star}}_{X}(s, 0) \leq V^{\star}_{M_{\overline{H}}}(s, 0) - \frac{\epsilon}{4}$. If we now apply Lemma~\ref{lem:XunbiasedclosetoX} using the specified value of $N$, we conclude that $V^{\pi^{\star}}_{X_{\text{unbiased}}}(s, 0) \leq V^{\star}_{M_{\overline{H}}}(s, 0) - \frac{\epsilon}{8}$. We upper-bound the probability of this consequent using Lemma~\ref{lem:dephoeffding}, after observing that $V^{\pi^{\star}}_{X_{\text{unbiased}}}(s, 0)$ can be rewritten as a convex combination over batches (from \eqref{eq:XAsBsum-s}), and its expected value is $V^{\pi^{\star}}_{M_{\overline{H}}}(s, 0)$ (from Lemma~\ref{lem:s-exp}). Defining $\delta^{\prime} \eqdef \frac{\delta}{|S||A|^{|S|}}$ and recalling that each batch has $N^{\prime} = N / \overline{H}$ worlds, it follows from Lemma~\ref{lem:dephoeffding} that the probability of $\pi^{\star}$ being under-estimated by at least $\frac{\epsilon}{2}$ for state $s$ is at most $\delta^{\prime}$.

Now fix state $s \in S$ and a policy $\pi: S \to A$ other than $\pi^{\star}$. By a symmetric working to the preceding one for $\pi^{\star}$, we obtain that the probability $\pi$ is over-estimated by at least $\frac{\epsilon}{2}$ on $s$ is again at most $\delta^{\prime}$. Since we have considered $|A|^{|S|}$ policies and $|S|$ states, a union bound restricts the mistake probability to $|S||A|^{|S|}\delta^{\prime} = 
\delta$.
% \subsection{Proof of Lemma~(\ref{lem:likelihood_ratio_bnd})}
\section{Proof of Lower Bound from Section~\ref{sec:lowerbound}}
\label{app:lowerboundproof}

In this section, we fill in the full proof of Theorem~\ref{thm:lower_bound_stationary}.
%In the remainder of this section, 
First we describe a family of MDPs constructed to achieve this lower bound, and then follow with the proof. Both the construction and the proof are closely based on those from \citet{Azar2013GenModel}.
%Some technical aspects of the proof are deferred to Appendix~\ref{app:lowerboundproof}.

%We prove this theorem using the approach in ~\citep{Azar2013GenModel}, but with the following differences:
%\begin{enumerate}
%    \item we prove the result for MDPs with a finite horizon $H$ and no discount factor, in contrast to the infinite horizon and discount factor $\gamma < 1$ assumed in ~\citep{Azar2013GenModel}
%    \item our approach does not make any assumption on the algorithm. The algorithm can be random and the output for a state-action pair $\hat{Q}(z)$ need not be independent of the output for other pairs $\hat{Q}(z')$. To this end, we construct a family of $\frac{N}{3}+1$ MDPs instead of just $2$ MDPs.
%\end{enumerate}

\subsection{Family of MDPs}

% \begin{wrapfigure}{l}{0.5\textwidth}
\begin{figure}[b]
\vspace{-0.5cm}
\begin{center}
\begin{tikzpicture}[scale=0.4,
every edge/.style={
        draw
        % postaction={decorate,
        %             decoration={markings,mark=at position 0.5 with {\arrow[scale=2]{>}}}
        %            }
        },
every loop/.style={},
el/.style = {inner sep=2pt, align=left, sloped},
]
\pgfdeclaredecoration{ignore}{final}
{
\state{final}{}
}
\pgfdeclaremetadecoration{middle}{initial}{
    \state{initial}[
        width={(\pgfmetadecoratedpathlength - \the\pgfdecorationsegmentlength)/2},
        next state=middle
    ]
    {\decoration{moveto}}

    \state{middle}[
        width={\the\pgfdecorationsegmentlength},
        next state=final
    ]
    {\decoration{curveto}}

    \state{final}
    {\decoration{ignore}}
}
\pgfdeclaremetadecoration{left}{initial}{
    \state{initial}[
        width={(\pgfmetadecoratedpathlength - \the\pgfdecorationsegmentlength)/10},
        next state=middle
    ]
    {\decoration{moveto}}

    \state{middle}[
        width={\the\pgfdecorationsegmentlength},
        next state=final
    ]
    {\decoration{curveto}}

    \state{final}
    {\decoration{ignore}}
}
\large
    \tikzset{state/.style={circle,draw=black,minimum size=8mm,thick}}
    \node[state] at (0, 0) (s0){$x_i$};
    \node[state] at (4, 8) (s01){$y_{i1}$};
    \node[state] at (4, 3) (s02){$y_{i2}$};
    \node[state] at (4, -6) (s03){$y_{iL}$};
    \node[state] at (11, 8) (s11){$y^2_{i1}$};
    \node[state] at (11, 3) (s22){$y^2_{i2}$};
    \node[state] at (11, -6) (s33){$y^2_{iL}$};
    \coordinate (s1) at (0, 8);
    \coordinate (s2) at (0, -6);
    \path[->]
    (s0) edge [thick] node[left,pos=0.8] {$a_{1}$}  (s01);
    \path[->]
    (s0) edge [thick] node[left,pos=0.8] {$a_{2}$}  (s02);
    \path[->]
    (s0) edge [thick] node[left,pos=0.8] {$a_{L}$}  (s03);
    \path[->]
    (s01) edge [loop above, thick] node[right,pos=0.7] {$p_{i1}$}  (s01);
    \path[->]
    (s02) edge [loop above, thick] node[right,pos=0.7] {$p_{i2}$}  (s02);
    \path[->]
    (s03) edge [loop above, thick] node[right,pos=0.7] {$p_{iL}$}  (s03);
    \path[->]
    (s01) edge [thick] node[below] {$1 - p_{i1}$}  (s11);
    \path[->]
    (s02) edge [thick] node[below] {$1 - p_{i2}$}  (s22);
    \path[->]
    (s03) edge [thick] node[below] {$1 - p_{iL}$}  (s33);
    \path[->]
    (s02) edge [-, thick, dotted, decorate, decoration={middle}, segment length=10mm] node[left,pos=0.8] {}  (s03);
    \path[->]
    (s22) edge [-, thick, dotted, decorate, decoration={middle}, segment length=10mm] node[left,pos=0.8] {}  (s33);
    \path[->]
    (s1) edge [-, thick, dotted, decorate, decoration={left}, segment length=15mm] node[left,pos=0.8] {}  (s0);
    \path[->]
    (s2) edge [-, thick, dotted, decorate, decoration={left}, segment length=10mm] node[left,pos=0.8] {}  (s0);
    % \path[->]
    % (s11) edge [-, thick] node[left] {$a_{1}$}  (s21);
    % \path[->]
    % (s11) edge [-, thick] node[right] {$a_{2}$}  (s22);
    % \path[->]
    % (s12) edge [-, thick] node[left] {$a_{1}$}  (s23);
    % \path[->]
    % (s12) edge [-, thick] node[right] {$a_{2}$}  (s24);
\normalsize
  \end{tikzpicture}
\end{center}
\vspace{-0.5cm}
\caption{The MDP setup used to derive the lower bound.}

% \Description{A diagram of the construction of the MDP. We use slight modifications of this MDP structure to prove our lower bound.}
\label{fig:mdp_construction}
%\Description{MDP family used in lower bound.}
\vspace{-0.5cm}
\end{figure}
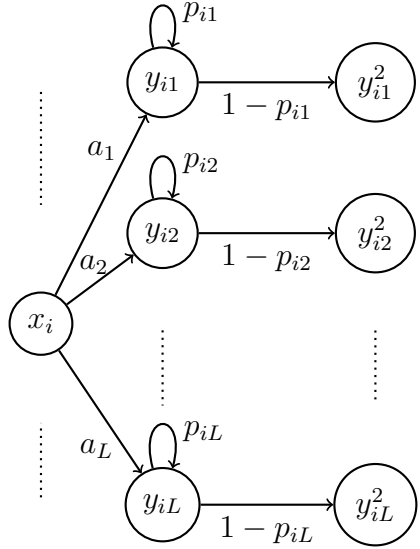
% \end{wrapfigure}

% \begin{figure}
%     \centering
%     \includegraphics[width=0.8\linewidth]{images/sk_mdp_bound.drawio.png}
%     \caption{The MDP setup used to derive the bounds}
%     \Description{A diagram of the construction of the MDP. We slight modifications of this MDP structure to prove our lower bound.}
%     \label{fig:enter-label}
% \end{figure}

Consider an MDP $M_{ab}(p, \alpha, H)$ defined as follows (see Figure~\ref{fig:mdp_construction}).
\begin{enumerate}
    \item $K$ initial states $X = \{x_i : i \in [1, K]\}$.
    \item $L$ initial actions $A = \{ a_j : j \in [1, L] \}$ from each initial state $x_i$. These actions are deterministic and yield a reward of $0$. Taking action $a_j$ from state $x_i$ leads us to state $y_{ij}$.
    \item $KL$ secondary states $Y^1 = \{ y_{ij} : i \in [1,K], j \in [1,L] \}$, each with only one action. This is the only set of states  in the MDP with a non-deterministic action. With probability $p_{ij}$ we stay at the same state, and with $1-p_{ij}$ we go to a corresponding state in $Y^2= \{ y^2_{ij} : i \in [1,K], j \in [1,L] \}$. The reward of this action is always $1$.
    \item \[ p_{ij}= 
        \begin{cases}
            p+\alpha,& \text{if } i = a, j = b,\\
            p,  & \text{otherwise.}
        \end{cases} \]
    \item $KL$ ``terminal'' states $Y^2$, each with only one action---looping back to the state and having a reward of $0$ (not shown).
    \item A horizon of $H$, after which the episode ends.
    \item No discounting of rewards (that is, $\gamma = 1$).
\end{enumerate}
We also define another MDP $M_0(p, \alpha, H)$ where $p_{ij} = p$ for all $i,j$, and then define the set of MDPs
\[\mathcal{M}(p, \alpha, H) = \{M_0(p, \alpha, H)\} \cup \{ M_{ab}(p, \alpha, H) : a \in [1,K], b \in [1,L] \}.\] We use $i \in [1,KL]$ interchangeably with $\{i,j\} : i \in [1,K], j \in [1,L]$. Note that the total number of state-action pairs is $k \eqdef 3KL$.

We establish some properties of our family of MDPs before proceeding to the proof of Theorem~\ref{thm:lower_bound_stationary}.

\begin{lemma}
    \label{lem:mdp_closed_form}
    The optimal value function for MDPs in the class $\mathcal{M}(p, \alpha, H)$ is given by:
    \begin{align*}
        V_i^*(y_j, 0) &= \begin{cases}
                \dfrac{1 - (p+\alpha)^H}{1-(p+\alpha)},& \text{if } i = j,\\
                \dfrac{1 - p^H}{1-p},  & \text{otherwise.}
            \end{cases} \numberthis \label{eqn:mdp_closed_form}
    \end{align*}
\end{lemma}
    This fact is easily verified. Note that the definition also holds for $M_0$ because $j$ is never equal to $0$.

\begin{lemma}
    \label{lem:MDP_distance_two_eps}
    The optimal value functions for the MDPs $M_i(p = 1 - \frac{1}{H}, \alpha = \frac{40\epsilon}{H^2}, H)$ and $M_0(p = 1 - \frac{1}{H}, \alpha = \frac{40\epsilon}{H^2}, H)$ (with $\epsilon<1, H>200$) are far apart: for $i \in [1, KL]$,
    \[ Q^*_i(y_i) - Q^*_0(y_i) > 2\epsilon. \]
    %This is a mathematical result that we derive in Appendix~\ref{app:lowerboundproof}.
\end{lemma}

%\subsection{Proof sketch for the lower bound}
%To derive the lower bound, we construct the aforementioned family of MDPs that share similar data observation probabilities but have significantly different Q-values. By comparing any MDP $M_i$ in this family to $M_0$, which differs only at the $i$-th state-action pair, we can lower-bound the ratio of the likelihood of observing the same data from these MDPs in terms of the number of samples. If the sample complexity is ``small'', the likelihood ratio bound allows us to show that if a certain event has a high probability under MDP $M_0$, it also has a high probability under $M_i$. By choosing this event to be the event that $|\hat{Q} - Q_0^*| < \epsilon$,
%where $\hat{Q}$ is the estimated action-value function and $Q_0$ is the true optimal action-value function of $M_0$, we can argue that if an algorithm gives an $\epsilon$-correct answer with probability at least $1 - \delta$ on $M_{0}$, it gives this same (but now \textit{not} $\epsilon$-correct) answer on $M_{i}$ with probability at least $1 - \delta$. Hence an algorithm cannot be PAC unless the sample complexity is ``large''---as quantified in Theorem~\ref{thm:lower_bound_stationary}.

%claimed.

%is PAC on MDP $M_0$ (i.e., it can estimate $\hat{Q}$ with high probability), then it cannot be PAC on MDP $M_i$. This contradiction establishes the lower bound on the sample complexity.

\begin{proof}
%\subsection{Proof of Lemma~\ref{lem:MDP_distance_two_eps}}
\begin{align*}
    (1-p)\left( (p+\alpha)^H - p^H \right) &= \frac{1}{H}\left( (p+\alpha)^H - p^H \right)\\
    &= \frac{1}{H} \left( (p+\alpha) - p \right) \left( (p+\alpha)^{H-1} + (p+\alpha)^{H-2}p + \ldots + p^{H-1} \right)\\
    &\leq \frac{\alpha}{H} H (p+\alpha)^{H-1} = \alpha (p+\alpha)^{H-1}. \numberthis \label{eqn:mdps_far_part1}
    % &= \alpha \left( 1 - \frac{1}{H} + \frac{2\epsilon (1-p)^2}{\gamma p} \right)^{H-1}\\
    % &< \alpha \left( 1 - \frac{1}{H} + 8 (1-p)^2 \right)^{H-1} \left( \because \epsilon < 1, \gamma > \frac{1}{2}, p > \frac{1}{2} \right)\\
    % &= \alpha \left( 1 - \frac{1}{H} + \frac{8}{H^2} \right)^{H-1}\\
    % &< \frac{2\alpha}{3} \left( \because \left( 1 - \frac{1}{H} + \frac{8}{H^2} \right)^{H-1} < \frac{2}{3} \ \forall H > 15 \right)
\end{align*}
Moreover, it can be shown that $\left(1 - \frac{1}{n} + \frac{40}{n^2} \right)^{n-1}$ is a series that decreases monotonically for $n > 5$ and converges to $\frac{1}{e}$. Evaluating for $n=200$ gives us that $\left(1 - \frac{1}{H} + \frac{40}{H^2} \right)^{H-1} < \frac{1.25}{e} \ \forall H > 200$. Therefore,
\begin{align*}
     p^H < p^{H-1} &< (p + \alpha)^{H-1}\\
    &= \left(1 - \frac{1}{H} + \frac{40\epsilon}{H^2} \right)^{H-1}\\
    &<  \left(1 - \frac{1}{H} + \frac{40}{H^2} \right)^{H-1} < \frac{1.25}{e}.
\end{align*}
Now,
\begin{align*}
    Q_i^*(y_i) - Q_0^*(y_i) &= \frac{1 - (p + \alpha)^H}{(1 - (p + \alpha))} - \frac{1 - p^H}{1 - p}\\
    &= \dfrac{(1 - p) - (1 - p)(p + \alpha)^H - (1 - (p + \alpha)) + (1 - (p + \alpha))p^H}{(1-p)(1 - (p + \alpha))}\\
    &= \dfrac{\alpha - (1 - p)(p + \alpha)^H + (1 - (p + \alpha))p^H}{(1-p)(1 - (p + \alpha))}\\
    &= \dfrac{\alpha - (1 - p)(p + \alpha)^H + (1 - p)p^H - \alpha p^H}{(1-p)(1 - (p + \alpha))}\\
    &= \dfrac{\alpha - (1 - p)((p + \alpha)^H - p^H) - \alpha p^H}{(1-p)(1 - (p + \alpha))}\\
    &\geq \dfrac{\alpha - \alpha (p+\alpha)^{H-1} - \alpha p^H}{(1-p)(1 - (p + \alpha))}\\
    &> \dfrac{\alpha (1 - \frac{2.5}{e})}{(1-p)(1 - (p + \alpha))}\\
    &> \dfrac{\alpha \frac{1}{20}}{(1-p)(1 - (p + \alpha))}\\
    &> \dfrac{\frac{\alpha}{20}}{(1-p)^2}\\
    &= 2\epsilon.
\end{align*}
\end{proof}

\subsection{Proof of Theorem~\ref{thm:lower_bound_stationary} (Lower Bound)}
Fix $H>200$, $\epsilon<1$, $\frac{1}{2} < p = 1 - \frac{1}{H} < 1$, $0 < \alpha = \frac{40\epsilon}{H^2} < (1-p)/2$, $c_1' = 20$, $c_2' = 6$. Let random variable $\tau$ denote the sample complexity of our algorithm on the chosen instance. We prove the following statement:
\begin{align*}
    E[\tau] < &\frac{ c_1 k H^3}{\epsilon^2} \ln\left( \frac{c_2}{\delta} \right) \implies \exists M_i \in \mathcal{M}\left(1 - \frac{1}{H}, \frac{40\epsilon}{H^2}, H\right) : P_i(|Q_i - \hat{Q}| > \epsilon) > \delta.
\end{align*}
% Note that
% \begin{align*}
%     E[\tau] &\implies P\left( \left\{ \tau < \frac{ 2 c_1 k H^3}{\epsilon^2} \log\left( \frac{2 c_2}{\delta} \right) \right\} \right) \geq \frac{1}{2}
%     &\implies 
% \end{align*}
% from Markov's inequality.

\begin{lemma}[Chernoff Hoeffding's bound for Bernoulli random variables]
    \label{lem:hoeffding_bound}
    Define $s$ as the sum of $l$ i.i.d. Bernoulli($p$) tosses $\left(p > \frac{1}{2}\right)$. Define $\theta = \exp\left( - \frac{c_1' \alpha^2 l}{p(1-p)} \right)$, $\Delta = \sqrt{2p(1-p)l \ln \frac{c_2'}{2\theta}}$, and $\mathcal{E} = \{s \leq pl + \Delta \}$. Then:
    \begin{align*}
        P(\mathcal{E}) &> 1 - \frac{2\theta}{c_2'}.
    \end{align*}
    \begin{proof}
        \begin{align*}
            P(\mathcal{E}) &>1 - \exp\left( - \dfrac{KL(p+\Delta || p)}{l} \right)\\
            &\geq 1 - \exp \left( - \dfrac{\Delta^2}{2lp(1-p)} \right)\\
            &= 1 - \exp \left( - \ln \frac{c_2'}{2\theta} \right) = 1 - \frac{2\theta}{c_2'}.
        \end{align*}
    \end{proof}
\end{lemma}

\begin{lemma}[Change of measure using likelihood ratios]
    \label{lem:change_of_measure}
    \begin{equation}
        P_i(E) = E_i [ \mathbf{1}_E] = E_j \left[ \frac{L_i(W)}{L_j(W)} \mathbf{1}_E \right]
    \end{equation}
    where $W$ is a sample path (for example, the actual sequence of Bernoulli tosses) that controls the event $E$, and $L_i(w) = P_i(W = w)$ is the likelihood of observing this sample path under MDP $M_i$.
\end{lemma}
%\textcolor{blue}{This (F.2) needs explanation and/or a reference. Just add a line here.}
This is a well-known result from probability, used exactly in this form by \citet{Azar2013GenModel}.

\begin{lemma}[Bound on the ratio of likelihoods]
    \label{lem:likelihood_ratio_bnd}
    If $\alpha \leq (1-p)/2$, then, under $\mathcal{E}$, we have that
    \[ \dfrac{L_i(W)}{L_0(W)} \geq \left(\frac{2\theta}{c_2'}\right)^{{\frac{2}{c_1'}} + \frac{2(1-p)}{p c_1'} + 2\sqrt{\frac{2}{c_1'}}} \geq \frac{2\theta}{c_2'}. \]
    \begin{proof}
    We first note that $\alpha \leq \frac{1-p}{2} \leq p(1-p) \leq \frac{p}{2} \leq \frac{p}{\sqrt{2}}$ and $\alpha \leq \frac{1-p}{2} \implies \alpha^2 \leq \frac{(1-p)}{2}\frac{(1-p)}{2} \leq \frac{(1-p)}{2}\frac{1}{2} \leq \frac{p(1-p)}{2}$. Hence, $\alpha \leq (1-p)/2$ is a sufficient condition for the statements made in Eqn. \ref{eqn:likelihood_ratio_bnd_1}, Eqn. \ref{eqn:likelihood_ratio_bnd_2}, Eqn. \ref{eqn:likelihood_ratio_bnd_3}, Eqn. \ref{eqn:likelihood_ratio_bnd_4} and Eqn. \ref{eqn:likelihood_ratio_bnd_5} (derived below). The first inequality simply follows by combining these equations. The second inequality follows by observing that we chose $p>\frac{1}{2}$, $c_1' = 20$, $c_2' = 6$, which imply that $\frac{2}{c_1'} + \frac{2(1-p)}{p c_1'} + 2\sqrt{\frac{2}{c_1'}} \leq 1$ and that $\frac{2\theta}{c_2'} < 1$.

    Let $s$ be the sum of $l$ Bernoulli tosses. But there are two Bernoullis - $p$ and $p + \alpha$. Let $W$ be some sequence of tosses such that their sum was $l$. So we have:
    \begin{align*}
        \dfrac{L_i(W)}{L_0(W)} &= \dfrac{(p+\alpha)^s (1 -  p - \alpha)^{l-s}}{p^s (1-p)^{l-s}}\\
        &= \left(1 + \frac{\alpha}{p}\right)^s \left(1 - \frac{\alpha}{1-p}\right)^{l-s}\\
        &= \left(1 + \frac{\alpha}{p}\right)^s \left(1 - \frac{\alpha}{1-p}\right)^{l-\frac{s}{p}} \left(1 - \frac{\alpha}{1-p}\right)^{s\frac{1-p}{p}}.
      \end{align*}
    
    Now, if $\alpha < (1-p)/2$,
    \begin{align*}
        \left(1 - \frac{\alpha}{1-p}\right)^{\frac{1-p}{p}} &= \exp \left( \ln \left( \left(1 - \frac{\alpha}{1-p}\right)^{\frac{1-p}{p}} \right) \right)\\
        &= \exp \left( \frac{1-p}{p} \ln \left( \left(1 - \frac{\alpha}{1-p}\right) \right) \right)\\
        (a)&\geq \exp \left( \frac{1-p}{p} \left( - \frac{\alpha}{1-p} - \left( 
    \frac{\alpha}{1-p} \right)^2 \right) \right)\\
        &= \exp \left( - \left( \frac{\alpha}{p} + 
    \frac{\alpha^2}{p (1-p)} \right) \right)\\
        &= \exp \left( - \frac{\alpha}{p} \right) \exp \left( -
    \frac{\alpha^2}{p (1-p)} \right)\\
        (b)&\geq \left( 1 - \frac{\alpha}{p} \right) \left( 1 -
    \frac{\alpha^2}{p (1-p)} \right) \numberthis \label{eqn:likelihood_ratio_bnd_1}
    \end{align*}
    where (a) follows from $\ln(1 - u) \geq -u - u^2$ for $0 \leq u \leq 1/2$. Note that $\alpha < (1-p)/2 \implies \frac{\alpha}{1-p} \leq 1/2$. (b) follows from $\exp(-u) \geq 1 -u$ for $0 \leq u \leq 1$, and the observation that $\alpha < (1-p)/2 < 1-p < p \implies \frac{\alpha}{p} \leq 1$. 
    
    Hence, 
    \begin{align*}
        \dfrac{L_i(W)}{L_0(W)} &= \left(1 + \frac{\alpha}{p}\right)^s \left(1 - \frac{\alpha}{1-p}\right)^{l-\frac{s}{p}} \left(1 - \frac{\alpha}{1-p}\right)^{s\frac{1-p}{p}}\\
        &\geq \left(1 + \frac{\alpha}{p}\right)^s \left(1 - \frac{\alpha}{1-p}\right)^{l-\frac{s}{p}} \left( 1 - \frac{\alpha}{p} \right)^s \left( 1 - \frac{\alpha^2}{p (1-p)} \right)^s\\
        &= \left(1 - \frac{\alpha^2}{p^2}\right)^s \left(1 - \frac{\alpha}{1-p}\right)^{l-\frac{s}{p}} \left( 1 - \frac{\alpha^2}{p (1-p)} \right)^s\\
        &\geq \left(1 - \frac{\alpha^2}{p^2}\right)^l \left(1 - \frac{\alpha}{1-p}\right)^{l-\frac{s}{p}} \left( 1 - \frac{\alpha^2}{p (1-p)} \right)^l. \numberthis \label{eqn:likelihood_ratio_bnd_2}
    \end{align*}
    
    For $\alpha^2\leq p^2/2$ and $\alpha^2 \leq p(1-p)/2$, we have:
    \begin{align*}
        \left(1 - \frac{\alpha^2}{p^2}\right)^l &= \exp \left( \ln \left( \left(1 - \frac{\alpha^2}{p^2}\right)^l \right) \right)\\
        &= \exp \left( l \ln \left( 1 - \frac{\alpha^2}{p^2} \right) \right)\\
        (a)&\geq \exp \left( -2l \frac{\alpha^2}{p^2} \right)\\
        &\geq \left(\frac{2\theta}{c_2'}\right)^{\frac{2(1-p)}{p c_1'}} \numberthis \label{eqn:likelihood_ratio_bnd_3}
    \end{align*}
    where (a) follows from $\ln(1-u) \geq -2u$ for $0 \leq u \leq 1/2$, and the final step follows from the definition of $\theta$ and choice of $c_2'$:
    \begin{align}
        \exp \left( -2 \frac{l\alpha^2}{p(1-p)} \right) &= \theta ^ \frac{2}{c_1'} \geq \left(\frac{2\theta}{c_2'}\right)^{\frac{2}{c_1'}}  \label{eqn:theta_defn_bound_1} \\
        \exp \left( -2 \frac{l\alpha^2}{p^2} \right) &= \theta ^ \frac{2(1-p)}{c_1' p} \geq \left(\frac{2\theta}{c_2'}\right)^{\frac{2(1-p)}{p c_1'}} \label{eqn:theta_defn_bound_2}
    \end{align}
    Similarly,
    \[ \left( 1 - \frac{\alpha^2}{p (1-p)} \right)^l \geq \left(\frac{2\theta}{c_2'}\right)^{\frac{2}{c_1'}}. \numberthis \label{eqn:likelihood_ratio_bnd_4} \]
    
    Finally, under $\mathcal{E}$, and for $\alpha \leq (1-p)/2$, we have:
    \begin{align*}
        l - \frac{s}{p} &\leq \frac{\Delta}{p}. \text{ Therefore}\\
        \left(1 - \frac{\alpha}{1-p}\right)^{l-\frac{s}{p}} &\geq \exp \left( -2\frac{\Delta}{p} \frac{\alpha}{1-p} \right)\\
        &= \exp \left( -2 \frac{\alpha \sqrt{2p(1-p)l \log \frac{c_2'}{2\theta}} }{p (1-p)} \right)\\
        &= \exp \left( -2 \sqrt{\frac{2l\alpha^2}{p(1-p)}\log \frac{c_2'}{2\theta}} \right)\\
        &\geq  \exp \left( -2 \sqrt{ - \log \left( \frac{2\theta}{c_2'}^{\frac{2}{c_1'}} \right) \log \frac{c_2'}{2\theta}} \right)\\
        &= \exp \left( -2 \sqrt{ \frac{2}{c_1'} \log \frac{c_2'}{2\theta} \log \frac{c_2'}{2\theta}} \right)\\
        &= \log \left( \left(\frac{2\theta}{c_2'}\right)^{2\sqrt{\frac{2}{c_1'}}} \right). \numberthis \label{eqn:likelihood_ratio_bnd_5}
    \end{align*}
    
    % ($exp(-x) \geq y \implies -x \geq log(y) \implies x \leq -log(y) \implies \sqrt{x} \leq \sqrt{-log(y)} \implies exp(-\sqrt{x}) \geq exp(-\sqrt{-log(y)})$)
    %\textcolor{red}{The paper misses this factor of 2}

    \end{proof}
\end{lemma}

\begin{lemma}
    \label{lem:each_state_must_be_sampled}
    If an algorithm is $(\epsilon, \delta)$-PAC, then,
    \begin{align}
        E[\tau_i] > \tau_i^*  \eqdef \frac{H^3}{64000 \epsilon^2}\log \left( \frac{1}{6\delta} \right),
    \end{align}
where random variable $\tau_{i}$ denotes the number of samples of state-action pair $i$.
\end{lemma}

\begin{proof}
    We choose $p = 1 - \frac{1}{H}$, $H > 200$, $\epsilon < 1$ and $\alpha = 40\epsilon (1-p)^2$ and construct the MDPs $M_i(p, \alpha, H), M_0(p, \alpha, H)$. Note that this choice gives us 
    \begin{align*}
        % p &> \frac{1}{2}\\
        \alpha &= \frac{40 \epsilon}{H^2} < \frac{40}{H^2} < \frac{1}{2H} = \frac{1 - p}{2}.
        % \frac{2}{c_1'} + \frac{2(1-p)}{p c_1'} + 2\sqrt{\frac{2}{c_1'}} &< \frac{6}{c_1'} + 2\sqrt{\frac{2}{c_1'}} < 1
    \end{align*}

    Now let us assume that there is some $(\epsilon, \delta)$-PAC algorithm $\mathcal{L}$ with
    \[ E[\tau_i] \leq \tau_i^*. \] We show that this leads to a contradiction.

    First we bound $\theta$ as follows:
    \begin{align*}
        \theta &= \exp\left( - \frac{c_1' \alpha^2 \tau_i}{p(1-p)} \right)\\
        &= \exp\left( - \frac{c_1' 1600 \epsilon^2 (1 - p)^4 \tau_i}{p(1-p)} \right)\\
        &> \exp\left( - \frac{c_1' 3200 \epsilon^2 \tau_i}{H^3} \right) (\because p > 1/2).
    \end{align*}
    This gives us (by applying Markov's inequality):
    \[ \tau_i \leq 10\tau_i^* \implies \frac{\theta}{c_2'} > \delta. \]
    
    Let $E_i = \{ | Q_i^* - \hat{Q} | \leq \epsilon \}$.
    \[ \therefore P_0(E_0^c) \leq \delta, P_i(E_i^c) \leq \delta, \] by the assumption that $\mathcal{L}$ is $(\epsilon, \delta)$-PAC. Further, let $ E_i' = E_0 \cap \mathcal{E} \cap \left\{ \tau_i \leq 10\tau_i^* \right\}$.
        \begin{align*}
         \therefore P_0(E_i') &\geq P_0(E_0)P_0(\mathcal{E})P_0\left( \left\{ \tau_i \leq 10\tau_i^* \right\} \right)\\
         &\geq (1 - \delta) \left(1 - \frac{2\theta}{c_2'}\right) \frac{9}{10} \geq  \left(1 - \frac{\theta}{c_2'}\right) \left(1 - \frac{2\theta}{c_2'}\right) \frac{9}{10}\\
         &= \frac{5}{6}\frac{4}{6}\frac{9}{10} = \frac{1}{2},
    \end{align*} where we invoke Lemma \ref{lem:hoeffding_bound} to bound $P_0(\mathcal{E})$, Markov's inequality to bound $P_0\left( \left\{ \tau_i \leq 10\tau_i^* \right\} \right)$ and the fact that $\theta < 1$ by definition.
    Then, using Lemma~\ref{lem:likelihood_ratio_bnd} and Lemma~\ref{lem:change_of_measure} gives us:
     \begin{align*}
        P_i(E_0) &\geq P_i(E_i') = E_0\left[ \frac{L_i(W)}{L_0(W)} \mathbf{1}_{E_i'} \right]\\
        &\geq E_0\left[ \frac{2\theta}{c_2'} \mathbf{1}_{E_i'} \right] \geq E_0\left[ 2\delta \mathbf{1}_{E_i'} \right]\\
        &= 2\delta P_0(E_i') \geq \delta.
     \end{align*}
    
    Using Lemma~(\ref{lem:MDP_distance_two_eps}) gives us:
    \begin{align*}
        E_0 \subset E_i^c &\implies P_i(E_i^c) > P_i(E_0) \geq \delta,
    \end{align*}
    which contradicts that $\mathcal{L}$ is $(\epsilon, \delta)$-PAC.
    
    % Then, by Markov's inequality, we have that
    % \[ P\left( \left\{ \tau_i \leq \frac{H^3}{6400 \epsilon^2}\log \left( \frac{1}{6\delta} \right) \right\} \right) \geq \frac{1}{2}.\]
    % $\theta$ can be bounded as follows:
    % \begin{align*}
    %     \theta &= \exp\left( - \frac{c_1' \alpha^2 \tau_i}{p(1-p)} \right)\\
    %     &= \exp\left( - \frac{c_1' 1600 \epsilon^2 (1 - p)^4 \tau_i}{p(1-p)} \right)\\
    %     &> \exp\left( - \frac{c_1' 3200 \epsilon^2 \tau_i}{H^3} \right) \hfill (\because p > 1/2).
    % \end{align*}
    % \begin{align*}
    %     \therefore \tau_i \leq \frac{H^3}{6400 \epsilon^2}\log \left( \frac{1}{6\delta} \right) &\implies \frac{\theta}{c_2'} > \delta\\
    %     &\implies P_0(E_0) \geq 1 - \frac{\theta}{c_2'}\\
    %     &\implies P_i(E_i^c) \geq \frac{\theta}{c_2'} > \delta
    % \end{align*}
    % Rearranging the terms completes the proof.
\end{proof}

By constructing the family $\mathcal{M}(p, \alpha, H)$, we can extend Lemma~\ref{lem:each_state_must_be_sampled} to each $\tau_i$ and hence obtain a lower bound on the number of samples any $(\epsilon, \delta)$-PAC algorithm must observe. Note that $|\mathcal{M}| = KL = \frac{k}{3}$.

%%%%%%%%%%%%%%%%%%%%%%%%%%%%%%%%%%%%%%%%%%%%%%%%%%%%%%%%%%%%%%%%%%%%%%%%
\end{document}